\documentclass[letterpaper]{article}
\usepackage{arxivstyle} 

\usepackage[utf8]{inputenc}
\usepackage{amsmath,bm,amsfonts,amssymb,xfrac, amsthm}
\usepackage{graphicx}
\usepackage{xcolor}
\usepackage{xspace}
\usepackage{hyperref}
\usepackage{mathtools}
\usepackage[noend]{algpseudocode}
\usepackage{algorithm}
\usepackage{setspace}
\usepackage{enumitem}
\usepackage{longtable}

\usepackage{footnote}
\usepackage{footmisc}

\newcommand{\astfootnote}[1]{
\let\oldthefootnote=\thefootnote
\setcounter{footnote}{0}
\renewcommand{\thefootnote}{\fnsymbol{footnote}}
\footnote{#1}
\let\thefootnote=\oldthefootnote
}

\newcommand{\etal}{et al.\xspace}

\newcommand{\ie}{i.e.,\xspace}

\renewcommand{\etal}{et al.\xspace}
\renewcommand{\paragraph}[1]{\vspace{3mm}\noindent\textbf{#1}}

\newcommand{\shorten}[2]{#2}

\renewcommand{\b}[1]{{\bm{#1}}}   %

\newcommand{\1}{\b{1}}              %
\newcommand{\0}{\b{0}}              %
\newcommand{\x}{\b{x}}

\newcommand{\w}{\b{w}}
\newcommand{\y}{\b{y}}
\renewcommand{\u}{\b{u}}

\newcommand{\eps}{\bm{\varepsilon}} 
\renewcommand{\L}{\b{L}}            %
\newcommand{\U}{\b{U}}              %
\newcommand{\bSigma}{\b{\Sigma}}    %
\newcommand{\bLambda}{\b{\Lambda}}  %
\newcommand{\bOmega}{\b{\Omega}}  %
\newcommand{\bGamma}{\b{\Gamma}}  %

\renewcommand{\H}{{\b{H}}}
\newcommand{\I}{\b{I}}
\newcommand{\A}{\b{A}}
\newcommand{\B}{\b{B}}
\newcommand{\X}{\b{X}}

\newcommand{\W}{\b{W}}              %
\renewcommand{\O}{O}

\newcommand{\LG}{\L_{\hspace{-1px}G}}              %
\newcommand{\LJ}{\L_{\hspace{-1px}J}}              %
\newcommand{\LT}{\L_{\hspace{-1px}T}}              %

\newcommand{\UG}{\U_{\hspace{-1px}G}}              %
\newcommand{\UJ}{\U_{\hspace{-1px}J}}              %
\newcommand{\UT}{\U_{\hspace{-1px}T}}              %

\newcommand{\GFT}[1]{\textrm{GFT}\hspace{-.0mm}\{#1\}}

\newcommand{\DFT}[1]{\textrm{DFT}\hspace{-.0mm}\{#1\}}
\newcommand{\JFT}[1]{\textrm{JFT}\hspace{-.0mm}\{#1\}}

\newcommand{\Rbb}{\mathbb{R}}

\newcommand{\E}[1]{\mathbf{E}\hspace{-.3mm}\left[#1\right]}        %
\newcommand{\Es}[1]{\tilde{\mathbf{E}}\hspace{-.3mm}\left[#1\right]}        %
\newcommand{\var}[1]{\mathbf{Var}\hspace{-.3mm}\left[#1\right]}        %
\renewcommand{\vec}[1]{\textrm{vec}\hspace{-.5mm}\left(#1\right)}           %
\newcommand{\mat}[1]{\textrm{mat}\hspace{-.5mm}\left(#1\right)}           %
\newcommand{\diag}[1]{\textrm{diag}\hspace{-.5mm}\left(#1\right)}           %
\newcommand{\norm}[1]{\left\lVert#1\right\rVert}        %
\newcommand{\transpose}{\intercal}                      %
\newcommand{\hermitian}{*}                      %
\newcommand{\delequal}{\overset{\Delta}{=}} %

\newtheorem{theorem}{Theorem} \newtheorem{definition}{Definition}
\newtheorem{proposition}{Proposition} \newtheorem{lemma}{Lemma}
 \newtheorem{corollary}{Corollary}
\newtheorem{property}{Property}

\begin{document}

\title{Stationary time-vertex signal processing}

\author{
Andreas Loukas\textsuperscript{*1},
Nathana\"el  Perraudin\textsuperscript{*2}
\\ 
\textsuperscript{1}{Laboratoire de Traitement des Signaux 2, \'{E}cole Polytechnique F\'{e}d\'{e}rale Lausanne}\\
\textsuperscript{2}{Swiss Data Science Center, Eidgen\"ossische Technische Hochschule Z\"urich}\\
andreas.loukas@epfl.ch, 
nathanael.perraudin@sdsc.ethz.ch}
\fnsymbol{footnote}

\maketitle

\begin{abstract}
This paper considers regression tasks involving high-dimensional multivariate processes whose structure is dependent on some {known} graph topology. We put forth a new definition of time-vertex wide-sense stationarity, or \emph{joint stationarity} for short, that goes beyond product graphs. Joint stationarity helps by reducing the estimation variance and recovery complexity. In particular, for any jointly stationary process (a) one reliably learns the covariance structure from as little as a single realization of the process, and (b) solves MMSE recovery problems, such as interpolation and denoising, in computational time nearly linear on the number of edges and timesteps. Experiments with three datasets suggest that joint stationarity can yield accuracy improvements in the recovery of high-dimensional processes evolving over a graph, even when the latter is only approximately known, or the process is not strictly stationary.
\end{abstract}

\paragraph{Keywords:}
stationarity,
multivariate time-vertex processes,
harmonic analysis,
graph signal processing,
PSD estimation
{\color{white} \astfootnote{A. Loukas and N. Perraudin contributed equally to this work.}
}

\section{Introduction}

One of the main challenges when modeling multivariate processes is to decouple the estimation variance from the problem size. Consider an $N$-variate process unfolding over $T$ timesteps. If only mild assumptions are made then the number of realizations needed to reliably estimate the first two moments is up to a logarithmic factor proportional to $\O(NT)$, i.e., the data size~\cite{rudelson1998}.
Assuming that the process is time wide-sense stationarity (TWSS) makes the length $T$ of the process inconsequential. This is ideal for the univariate setting as it enables us to make relevant predictions even based on a single realization. 
If one additionally assumes that the signal autocorrelation is compactly supported, such that most data dependencies take place within a short time horizon, then the estimation variance can be reduced further by (roughly) splitting the observations into parts and considering each as an independent realization. This approach suffices when $N$ is relatively small.
For high-dimensional processes, however, one needs to incorporate additional assumptions to obtain meaningful predictions~\cite{lutkepohl2005new, ledoit2004well, lam2012factor, connor1995three}.

In this spirit, this paper focuses on high-dimensional processes that are supported on the vertex set and are statistically dependent on the edge set of some known graph topology.  
Whether examining epidemic spreading~\cite{keeling2005networks}, how traffic evolves in the roads of a city~\cite{mohan2008nericell}, or neuronal activation patterns present in the brain~\cite{huang2015graph}, many of the high-dimensional processes one encounters are inherently constrained by some underlying network. 
This realization has been the driving force behind recent efforts to re-invent classical models by taking into account the graph structure, with advances in many problems, such as denoising~\cite{zhang2008graph} and semi-supervised learning~\cite{smola2003kernels,belkin2004semi}, among others. 

Yet, standard models for processes (evolving) on graphs often fail to produce useful results when applied to real datasets. One of the main reasons for this shortcoming is that they model only a limited set of spatiotemporal behaviors. The well-used graph Tikhonov and total variation priors, for instance, assume that the signal varies slowly or in a piece-wise constant manner over edges, without specifying any precise relations~\cite{shuman2013emerging,sandryhaila2013discrete, sandryhaila2014big}. Similarly, assuming that the graph Laplacian encodes the conditional correlations of variables, as is done with Gaussian Markov Random Fields~\cite{gadde2015probabilistic}, becomes a rigid model when the graph is known~\cite{zhang2015graph}. To capture the behavior of complex networked systems, such as transportation and biological networks, it is crucial to train expressive models, being able to reproduce a wide range of graph and temporal behaviors.

\subsection{Contributions}

This paper considers the statistical modeling of processes evolving on graphs. 
In particular, we investigate the relationship between two different hypotheses: TWSS and VWSS~\cite{perraudin2016stationary,girault2015stationary,marques2016stationary}, which are individually helpful in reducing the variance of covariance estimation for time series and graph signals, respectively.
We propose a combined multivariate hypothesis that we refer to as time-vertex wide-sense stationarity, or \emph{joint stationarity} for short.
The necessary first step of our analysis consists of reformulating the standard properties of stationarity (such as the relation of the covariance matrix and power spectral density to an appropriate Fourier transform) from the lens of time-vertex analysis~\cite{isufi2017autoregressive,loukas2016frequency}.
This analysis is purposeful, yet not trivial as joint stationarity is more complicated than assuming stationarity on the product of two graphs\footnote{As it will be discussed in Section~\ref{subsec:relations}, joint stationarity is strictly more general than vertex stationarity on the product of the two graphs (first proposed in~\cite{sandryhaila2014big} in the deterministic setting), as the latter can only model processes with specific PSD.}.

We use the hypothesis of joint stationarity to control variance and computational complexity in estimation and recovery tasks.
Similar to~\cite{perraudin2016stationary}, also here one may reliably estimate the model parameters from few observations (e.g., see Figure~\ref{fig:covariance_estimation_comparison}) and solve MMSE recovery problems in time linear on the number of edges and timesteps (e.g., see Figure~\ref{fig:scalability}). 
Complimenting previous work, we also provide an analysis of the Power Spectral Density (PSD) estimation, which brings insight into the inherent trade-off between bias and variance.
In addition, we experimentally demonstrate that assuming joint stationarity aids in recovery even when only an approximation of the graph is known, or the process is only approximately jointly stationary. These experiments corroborate that the joint stationarity hypothesis is a useful assumption, particularly in situations when the problem features a large number of variables but only a limited number of observations.   

To test the utility of joint stationarity, we apply our methods on three diverse datasets: (a) a meteorological dataset containing the hourly temperature of 32 weather stations over one month in Molene, France~\cite{girault2015stationary}, (b) a traffic dataset depicting high-resolution daily vehicle flow of 4 weekdays in the highways of Sacramento, and (c) simulated SIRS-type epidemics over Europe. 
Our experiments confirm that for high-dimensional processes evolving over graphs, assuming joint stationarity yields an improvement in recovery performance as compared to time- or vertex-based stationarity methods, even when the graph is only approximately known and the data violate the strict conditions of our definition.

\subsection{Related work}

There exists an extensive literature on multivariate stationary processes, developing the original work of Wiener~\etal~\cite{wiener1957prediction,wiener1958prediction}. The reader may find interesting Bloomfield's book~\cite{bloomfield2004fourier} focusing on spectral relations. We focus on two main approaches that relate to our work, graphical models and signal processing on graphs.

\emph{Graphical models.} In the context of graphical models, multivariate stationarity has been used jointly with a graph in the work of~\cite{bach2004learning,dahlhaus2003causality}. Though relevant, we note that there is a key difference of these models with our approach: we assume that the graph is given, whereas in graphical models the graph structure (or more precisely the precision matrix) is learned from the data. 
Knowing the graph allows us to search for more involved relations between the variables. As such, we are not restricted to the case that the conditional dependencies are given by the graph (and therefore that they are sparse), but allow non-adjacent variables to be conditionally dependent,  modeling a broader set of behaviors. We also note that our approach is eventually more scalable. We refer to~\cite{zhang2015graph} for elements of connections between graphical models and graph signal processing. 

\emph{Graph signal processing.} The idea of studying the stationarity of a random vector w.r.t.  a graph was first introduced in~\cite{girault2015stationary,girault2015signal} and then in~\cite{perraudin2016stationary,marques2016stationary}. While these contributions have different starting points, they both roughly propose the same definition. Another more recent contribution relating to stationarity on graphs in the context of PSD estimation is~\cite{chepuri2016subsampling}. Despite the relevance of these works, it is important to stress that the current paper is the first to consider a stationary hypothesis over graph signals varying in time. Moreover, the new results are non-trivial as they cannot be obtained by applying previous definitions on a product graph. In addition, some of the analysis presented here (particularly that of Section~
\ref{sec:psd}) is novel and can also be employed for the previously studied case of stationary graph signals. To make the connection with previous works transparent, in the following every technical result (e.g., Lemma, Theorem, Proposition) that emerges as a generalization of~\cite{girault2015stationary,girault2015signal,marques2016stationary} contains a reference in its heading pointing to the former claim. 

The forecasting of time-evolving signals on graphs was also considered in~\cite{loukas2016predicting,mei2015signal,ioannidis2018inference,8081618}. Nonetheless, there are several differences with these works, with the most important being that we define joint stationarity, and that we are not restricted to the causal case (where a process is reconstructed only from its past).  
Finally, it should be noted that some preliminary results of this work appeared in a conference paper~\cite{perraudin2016towards}. This work extends the conference paper in many directions. We refine the definition of joint stationarity and explore how it relates to other well-known stationarity hypotheses. We propose a new PSD estimator and provide a theoretical analysis of the bias and variance of the old and new PSD estimators. Additionally, we study the complexity of the proposed solution and evaluate its merit w.r.t. two new datasets.

\section{Preliminaries}
\label{sec:preliminaries}

\paragraph{General notation.} We use boldface symbols for matrices and vectors (e.g., $\A$ and $\b{a}$ respectively) and calligraphic symbols for sets (e.g., $\mathcal{V}$ and $\mathcal{E}$). Symbol $j$ denotes the imaginary unit, $\I_N$ is the $N\times N$ identity matrix, and $\b{1}_N$ is the all-ones vector of size $N$. %
We use brackets to index matrix elements and subscripts for matrix blocks: if $\A$ is of size $N_1\times N_2$ then $\A[n_1,n_2]$ is the element at the $n_1$-th row and $n_2$-th column and $\A_{n_1,n_2}$ is a (block) matrix. Vector $\b{a} = \vec{\A}$ (without subscript) is the vectorized representation of $\A$ and $\b{a}_{n}$ is its $n$-th column. Moreover, $\A^\transpose$ is its transpose and $\A^\hermitian$ is its transposed complex conjugate (meaning that $({\A}^{\hermitian}_T)^\transpose$ is the complex conjugate). 
If $\A$ is $N\times N$ Hermitian, its eigen-decomposition is generically written as $\A = \U \bLambda \U^\hermitian$, where $\U = [\u_1, \ldots, \u_N$] is a matrix having eigenvectors as columns and $\bLambda = \diag{\lambda_1, \ldots, \lambda_N}$ is the diagonal matrix of eigenvalues.
Symbols $h(\cdot), f(\cdot), g(\cdot)$ are reserved for scalar/matrix functions. A matrix function with a single argument takes as an input a symmetric matrix $\A$ and outputs $h(\A) = \U \diag{h(\lambda_1), \ldots, h(\lambda_N)} \U^\top$. 
The operator $\otimes$ denotes the Kronecker product. The Kronecker sum $\oplus$ can be defined in terms of the Kronecker product as $\A \oplus \B = \A \otimes \I_{M} + \I_{N} \otimes \B$, where  matrix $\B$ has size $M \times M$.

\paragraph{Harmonic time-vertex analysis.} 
We consider signals supported on the vertices $\mathcal{V} = \{ v_1, v_2, \ldots, v_N \}$ of a weighted undirected graph $\mathcal{G} = (\mathcal{V}, \mathcal{E}, \W_G)$, with $\mathcal{E}$ the set of edges of cardinality $E = |\mathcal{E}|$ and $\W_G$ the weighted adjacency matrix. 
Suppose that signal $\x_t$ is sampled at $T$ successive regular intervals of unit length. A real time-vertex signal $\X = \left[ \x_1, \x_2, \ldots, \x_T \right] \in \mathbb{R}^{N\times T}$ is then the matrix having graph signal $\x_t$ as its $t$-th column.  

The frequency representation of a time-vertex signal $\X$ is given by the Joint Fourier Transform~\cite{loukas2016frequency,sandryhaila2014big} (or JFT for short) 
\begin{align}
    \hat{\X} = \JFT{\X} \delequal \GFT{\DFT{\X}} = \UG^* \X ({{\U}^{\hermitian}_T})^\transpose,
\end{align}
with $\UG$ and $\UT$ being, respectively, the unitary Graph Fourier Transform (GFT) and Discrete Fourier Transform (DFT) matrices, whereas $({\U}^{\hermitian}_T)^\transpose$ is the complex conjugate of $\UT$. In vector form, we have that $\hat{\x} = \JFT{\x} \delequal \UJ^\hermitian \, \x$, where $\UJ = \UT \otimes \UG$. %
As is often the case, we choose $\UG$ to be the eigenvector matrix of the combinatorial\footnote{Though we use the combinatorial Laplacian in our presentation, our results can be adapted to alternative positive semi-definite matrix definitions of a graph Laplacian, such as the normalized Laplacian.} graph Laplacian matrix $\LG = \diag{\W_G \1_N} - \W_G$, where $\1_N$ is the all-ones vector of size $N$, and $\diag{\W_G \1_N}$ is the diagonal degree matrix. 
Matrix $\UT$ is the eigenvector matrix of the Laplacian $\LT$ of a cyclic graph $\mathcal{T}$:
\begin{align}\label{eq:dftmtx}
    \UT^{\hermitian}[\tau,t]= \frac{e^{-j \omega_\tau t}}{\sqrt{T}}, \quad \text{with} \quad \omega_\tau = \frac{2 \pi (\tau-1)}{T} \quad \text{for} \quad t,\tau = 1, 2, \ldots, T.
\end{align}
With this in place, $\hat{\X}[n,\tau]$ can be seen as the Fourier coefficient associated with the joint frequency $[\lambda_n, \omega_\tau]$, where $\lambda_n$ denotes the $n$-th graph eigenvalue and $\omega_\tau$ the $\tau$-th angular frequency.

The JFT maintains a close connection with the product graph $\mathcal{J}$~\cite{sandryhaila2014big,loukas2016frequency}. The latter is the graph whose adjacency matrix is $\b{W}_J = \b{W}_T \oplus \b{W}_G$ (this amounts to a Cartesian product between $\mathcal{G}$ and the ring graph $\mathcal{T}$). %
The connection is revealed if one realizes that the Laplacian $\LJ = \LT \oplus \LG$ of $\mathcal{J}$ carries the eigen-decomposition $\LJ = \UJ (\bLambda_T \oplus \bLambda_G) \UJ$. It follows that computing the JFT (in vector form) is the same as computing the GFT of $\x$ w.r.t. graph $\mathcal{J}$. The main issue with any\footnote{The same limitation holds for product graph constructions that do not rely on the Cartesian product~\cite{sandryhaila2014big}.} product graph interpretation is that it imposes a strict dependence between the eigenvalues of $\LG$ and $\LT$ (since the eigenvalues of $\LJ$ are given by $\bLambda_T \oplus \bLambda_G$). As we will see in the next paragraph, to attain full generality one needs to abandon the product graph. For an in-depth discussion of JFT and its properties, we refer the reader to~\cite{grassi2017timevertex}.

\paragraph{Joint time-vertex filtering.} Filtering a time-vertex signal $\x$ with a \emph{joint filter} $h(\LG,\LT)$ corresponds to element-wise multiplication in the joint frequency domain $[\lambda, \omega]$ by a function $h: [0, \lambda_{\max}] \times [-1,\ 1] \mapsto \Rbb$ ~\cite{grassi2017timevertex,loukas2015distributed,isufi2016separable,loukas2016frequency}. When a joint filter $h(\LG,\LT)$ is applied to $\x$, the output is
\begin{align} \label{eq:def_joint_filtering}
    h(\LG,\LT) \, \x &= \UJ\, h(\bLambda_G,\bOmega) \, \UJ^\hermitian \x,
\end{align}
where $\bLambda_G \in \Rbb^{N\times N} $ and $\bOmega  \in \Rbb^{T\times T}$ are diagonal matrices with $\bLambda_G[n,n] = \lambda_n$ and $\bOmega[\tau,\tau] = \omega_\tau $, whereas $h(\bLambda_G,\bOmega)$ is a diagonal $NT\times NT$ matrix defined as  
\begin{align}
     h(\bLambda_G,\bOmega) = 
     \text{diag}
     \bigg(  
        \text{vec} \hspace{-0.1mm}\bigg( 
        \begin{bmatrix}
 h(\lambda_1, \omega_1) & \cdots & h(\lambda_1, \omega_T) \\
 \vdots &  \ddots & \vdots \\
 h(\lambda_N, \omega_1)  & \cdots & h(\lambda_N, \omega_T)
        \end{bmatrix}
        \bigg)
        \bigg)
    \notag 
\end{align}
and $\diag{\vec{\A}}$ creates a matrix with diagonal elements the vectorized form of $\A$. 
The bi-variate notation $h(\cdot, \cdot)$ is meant to illustrate that joint filters operate \textit{independently} on the two domains, something impossible\footnote{Defining joint filters in terms of a product graph would imply that there is a fixed relation between angular and graph frequencies (determined by the product graph construction). As a result, in the product graph framework filters are univariate functions.} in the product graph framework~\cite{sandryhaila2014big,loukas2016frequency}. 
For convenience, we will often overload notation and write $h(\theta_{n,\tau})$ to refer to the bivariate function $h(\lambda_n, \omega_\tau)$.  
Furthermore, we say that a joint filter is \emph{separable}, if its joint frequency response $h$ can be written as the product of a frequency response $h_1$ defined solely in the vertex domain and one $h_2$ in the time domain, i.e., $h(\theta) = h_1(\lambda) \cdot h_2(\omega)$.

\section{Joint Time-Vertex Stationarity}
\label{sec:stationarity}

Let $\X \in \Rbb^{N\times T}$ be a real discrete periodic multivariate stochastic process with a finite number of timesteps $T$ that is indexed by the vertex $v_i$ of graph $\mathcal{G}$ and time $t$. We refer to such processes as {time-vertex processes}, or \emph{joint processes} for short. 

Our objective is to provide a definition of stationarity that captures statistical invariance of the first two moments of a joint process $\x = \vec{\X} \sim \mathcal{D}(\bar{\x}, \bSigma)$, \ie the mean $\bar{\x} = \E{\x}$ and the covariance $\bSigma = \E{\x\x^\transpose} - \bar{\x}\bar{\x}^\transpose$. Crucially, the definition should do so in a manner that is faithful to the graph and temporal structure.  

\subsection{Definition}

Typically, wide-sense stationarity is thought of as an invariance of the two first moments of a process w.r.t.  translation. For the first moment, things are straightforward: stationarity implies a constant mean $\E{\x} = c \1$, independently of the domain of interest. The second moment, however, is more complicated as it depends on the exact form translation takes in the particular domain. Unfortunately, for graphs translation is a non-trivial operation and three alternative translation operators exist: the generalized translation~\cite{shuman2016vertex}, the graph shift~\cite{sandryhaila2013discrete}, and the isometric graph translation~\cite{girault2015signal}. Due to this challenge, there are currently three alternative (though akin) definitions of stationarity appropriate for graphs~\cite{perraudin2016stationary,girault2015stationary,marques2016stationary}. %

The ambiguity associated with translation on graphs urges us to seek an alternative starting point for our definition.  
Fortunately, there exists an interpretation which holds promise:
\textit{up to its constant mean, a wide-sense stationary process corresponds to a white process filtered linearly on the underlying space}. This ``filtering interpretation'' of stationarity is well known classically\footnote{As the correlation between two instants $t_1$ and $t_2$ depends only on the difference between these two instants $\E{\x[t_1]\x[t_2]}-\E{\x[t_1]}\E{\x[t_2]}=\b{\gamma}[t_1-t_2]$, the covariance matrix has to be circulant, a property that is shared by linear filters.} as well as in the graph setting~\cite{marques2016stationary} and is equivalent to asserting that the second moment can be expressed as $\bSigma = h(\LT)$, where $h(\LT)$ is a linear filter. Thankfully, not only filtering is elegantly and uniquely defined for graphs~\cite{shuman2016vertex}, but also stating that a process is graph wide-sense stationary iff $\E{\x} = c \1_N$  and $\bSigma = h(\LG)$ is a graph filter, is generally consistent\footnote{The only exception: for graphs with repeated eigenvalues, the conditions $\E{\x} = c \1$  and $\bSigma = h(\LG)$ are sufficient but not necessary for the graph stationarity definition based on isometric graph translation~\cite{girault2015stationary}.} with current definitions~\cite{perraudin2016stationary,girault2015stationary,marques2016stationary}.

This motivates us to also express the definition of stationarity for joint processes in terms of joint filtering: 

\begin{definition}[JWSS] \label{def:time-vertex-stationarity}
A joint process $\x = \vec{\X}$ is called Jointly Wide-Sense Stationary (JWSS), if and only if
\begin{enumerate}[label=(\alph*)]
  \setlength\itemsep{1mm}
    \item The first moment of the process is constant $\E{\x} = c\1_{NT}$. 
    \item The covariance matrix of the process is a joint filter $ \bSigma = h(\LG, \LT)$, where $h(\cdot, \cdot)$ is a non-negative real function referred to as joint power spectral density (JPSD).
\end{enumerate}
\end{definition}

Let us examine Definition~\ref{def:time-vertex-stationarity} in detail.

\vspace{2mm}\emph{First moment condition.} As in the classical case, the first moment of a JWSS process has to be constant over the time and the vertex sets, i.e., $\bar{\X}[{i,t}] = c $ for every $i = 1,2,\ldots,N$ and $t = 1,2,\ldots,T$. For alternative choices of the graph Laplacian with a null-space not spanned by the constant vector, the first moment condition should be modified to requiring that the expected value of a JWSS process is in the null space of the matrix $\LT \oplus \LG$ (see Remark 2~\cite{marques2016stationary} for a similar observation on stochastic graph signals).

\vspace{2mm}\emph{Second moment condition.} According to the definition, the covariance matrix of a JWSS process takes the form of a joint filter $h(\LG,\LT)$, and is therefore diagonalizable by the JFT matrix $\UJ$. 
It may also be interesting to notice that the matrix $h(\LG,\LT)$ can be expressed as follows
\begin{equation} \label{eq:joint_filter_expand}
    \bSigma = h(\LG,\LT) = \left(\begin{array}{cccc}
\H_{1,1}& \H_{1,2}& \cdots & \H_{1,T}\\
\H_{2,1} & \H_{2,2} & & \H_{2,T} \\
\vdots &  &  \ddots & \vdots \\
\H_{T,1} & \H_{1,2}  &  \cdots & \H_{T,T}
\end{array}\right),
\end{equation}
where each block $\H_{t_1,t_2}$ of $\bSigma$ is an $N\times N$ matrix defined as:
\begin{equation} \label{eq:joint_filter_detail}
\H_{t_1,t_2} = \frac{1}{T} \sum_{\tau = 1}^{T}  h_{\omega_\tau} (\LG) 
\, e^{j \omega_\tau (t_1-t_2+1)}
\end{equation}
and $h_{\omega_\tau} (\LG)$ is the graph filter with frequency response $h_{\omega_\tau} = h(\lambda,\omega_\tau)$. 
Being a covariance matrix, $h(\LG,\LT)$ must necessarily be positive-semidefinite; thus $h(\cdot,\cdot)$ is real (the eigenvalues of every Hermitian matrix are real) and non-negative. Also equivalently, every zero mean JWSS process $\x = \vec{\X}$ can be generated by joint filtering $\x = h(\LG,\LT)^{1/2} \b{\eps}$ a white process $\b{\eps}$ with zero mean and identity covariance.  
The following proposition exploits these facts to provide an interpretation of JWSS processes in the joint frequency domain.

\begin{proposition}[Generalizes Theorem 1~\cite{perraudin2016stationary} and Proposition 1 ~\cite{girault2015stationary,marques2016stationary}] \label{theorem:time-vertex-stationarity}
A joint process $\X$ over a connected graph $\mathcal{G}$ is Jointly Wide-Sense Stationary (JWSS) if and only if:
\begin{enumerate}[label=(\alph*)]
  \setlength\itemsep{1mm}
    \item The joint spectral modes are in expectation zero 
    $$\E{\hat{\X}[n,\tau]}=0 \quad \text{if } \lambda_{n} \neq 0 \text{ and } \omega_{\tau} \neq 0. $$
    \item The product graph spectral modes are uncorrelated 
    $$\E{ \hat{\X}[n_1,\tau_1] \hat{\X}[n_2,\tau_2]} = 0, $$
    whenever ${n_1} \neq {n_2} \text{ or } \tau_1 \neq {\tau_2}$.
    \item There exists a non-negative function $h(\cdot,\cdot)$, referred to as joint power spectral density (JPSD), such that
    $$ \E{\left|\hat{\X}[n,\tau]\right|^2} - \left|\E{\hat{\X}[n,\tau]}\right|^2 = h(\lambda_n, \omega_\tau),$$
    for every $n = 1, 2, \ldots, N$ and $\tau = 1, 2, \ldots, T$.
\end{enumerate}
\end{proposition}

(For clarity, this and other proofs of the paper have been moved to the appendix.)

We briefly present a few additional properties of JWSS processes that will be useful in the rest of the paper. 

\begin{property}[Generalizes Example 1~\cite{perraudin2016stationary,girault2015stationary,marques2016stationary}]
White centered i.i.d. noise $\w \in \Rbb^{NT} \sim \mathcal{D}(\0_{NT},\I_{NT})$ is JWSS with constant JPSD for any graph. 
\end{property}

The proof follows easily by noting that the covariance of $\w$ is diagonalized by the joint Fourier basis of any graph $\bSigma_\w = \I = \UJ \I \UJ^*$. This last equation tells us that the JPSD is constant, which implies that similar to the classical case, the energy of white noise is evenly spread across all joint frequencies.   

A second interesting property of JWSS processes is that stationarity is preserved through a filtering operation. 
\begin{property}[Generalizes Theorem 2~\cite{perraudin2016stationary}, Property 1~\cite{marques2016stationary}] \label{theo:time-vertex-psd-trans}
When a joint filter $f(\LG, \LT)$ is applied to a JWSS process $\X$ with JPSD $h$, the result $\b{Y}$ remains JWSS with mean $c f(0,0)\1_{NT}$, where $c$ is the mean of $\X$, and JPSD $f^2(\lambda, \omega) \, h(\lambda, \omega).
$\end{property}

Finally, we notice that for real processes $\X$, which are the focus of this paper, the function $h$ forming the joint filter should be symmetric w.r.t. $\omega$, meaning that $h(\lambda,\omega) = h(\lambda,-\omega)$. This property can be easily derived from the definition of the Fourier transform.

\subsection{Relations to classical definitions}
\label{subsec:relations}

We next provide an in-depth examination of the relations between joint wide-sense stationarity, time and vertex stationarity, as well as their multivariate equivalents. 
For clarity, we order the rows/columns of the covariance matrix $\bSigma$ such that each $\bSigma_{t_1, t_2}$ block of size $N\times N$ measures the covariance between $\x_{t_1}$ and $\x_{t_2}$ (see~\eqref{eq:joint_filter_expand}).

\begin{figure}[t]
\centering

\includegraphics[width=0.8\columnwidth]{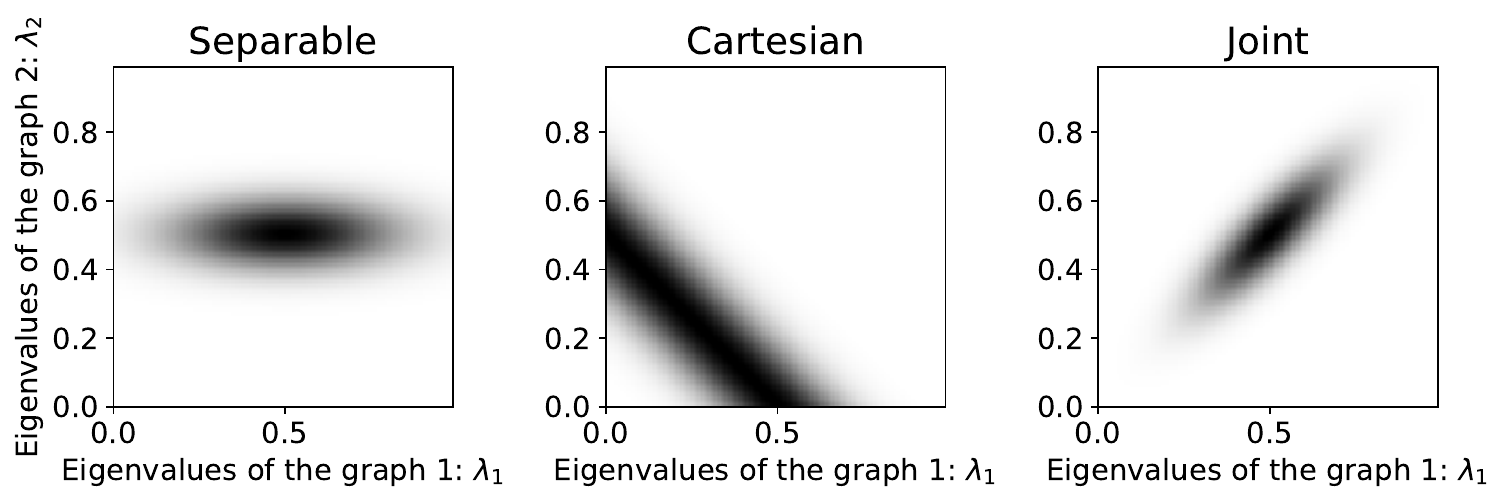}
\caption{The joint stationarity hypothesis is more general than assuming either (standard) VWSS and TWSS or VWSS on a (Cartesian) product graph.
The figure presents three examples of PSDs plotted as 2-dimensional function of $\lambda_1, \lambda_2$ that, for simplicity, corresponds to the eigenvalues of two graphs. The second graph (time) is a ring.
In the separable case (left), the PSD has to satisfy $h(\lambda_1,\lambda_2)=h_1(\lambda_1)h_2(\lambda_2)$, making it unable to capture any dependencies between $\lambda_1$ and $\lambda_2$. Using VWSS (middle) limits the PSD to $h(\lambda_1,\lambda_1)=h(\lambda_1+\lambda_2)$ leading to constant values along the diagonal line $\lambda_1+\lambda_2=c$. Joint stationarity (right) can encode any PSD $h(\lambda_1,\lambda_2)$, as exemplified here.}
\label{fig:comparizon_hypotheses}%
\end{figure}

\paragraph{Standard definitions}. As we discuss below, known definitions of stationarity in time/vertex domains are particular cases of joint stationarity.

\emph{TWSS $\cap$ VWSS $\subset$ JWSS.} The known versions of stationarity (TWSS, VWSS) are oblivious to any structure along one of the two dimensions of $\X$. In this manner, assuming that $\X$ is TWSS amounts to interpreting each of the $N$ time series as a separate realization of the \textit{same} process with TPSD $h_T(\omega)$. Similarly, if $\X$ is VWSS then each graph signal $\x_t$ is taken as a separate realization of a \textit{single} stochastic graph signal with VPSD $h_G(\lambda)$~\cite{perraudin2016stationary,marques2016stationary}.  It is a simple consequence that, different from the JWSS hypothesis, assuming that $\X$ is both TWSS and VWSS is equivalent to limiting our scope to separable JPSD defined as the product of two univariate functions $h(\lambda,\omega) = h_G(\lambda) h_T(\omega)$---see also Figure~\ref{fig:comparizon_hypotheses}.  

\paragraph{Definitions based on the product graph.}
As explained in Section~\ref{sec:preliminaries}, the JFT can be interpreted as a graph Fourier transform taken over a product graph whose Laplacian is $\LJ = \LG \oplus \LT$. This construction can give rise to two additional definitions for joint stationarity:

\emph{VWSS on a product graph.} The first is obtained by applying the VWSS definition of~\cite{perraudin2016stationary,marques2016stationary} on the graph associated with $\LJ$.
The resulting model is {not} sufficiently general in order to generate the full spectrum of JWSS processes. The reason is that, whereas the JPSD $h(\lambda, \omega)$ can be any two-dimensional non-negative function, the JPSD of any VWSS process on $\LJ$ is necessarily one-dimensional (the eigenvalues of $\LJ$ are the sums of all combinations of the eigenvalues of $\LG$ and $\LT$)---see Figure~\ref{fig:comparizon_hypotheses} for a pictorial demonstration and Appendix~\ref{app:univariate-vs-multivariate-PSD} for examples from real data. The same reasoning also holds for alternative products between graphs, such as the strong and Kronecker products~\cite{sandryhaila2014big}.

\emph{Covariance diagonalized by the product graph Fourier transform.} The second definition, which we refer to as JWSS-alternate, entails asserting that the covariance matrix $\bSigma$ can be diagonalized by the JFT, i.e., the eigenbasis of $\LJ$. This can be seen to differ from the JWSS definition only in case of graph Laplacian eigenvalue multiplicities: whenever the graph Laplacian features repeated eigenvalues, for Definition~\ref{def:time-vertex-stationarity} the degrees of freedom of the JPSD $h$ decrease, as necessarily $h(\lambda_1, \omega) = h(\lambda_2, \omega)$ when $\lambda_1 = \lambda_2$.  This restriction is motivated by the following observation: for an eigenspace with multiplicity greater than one, there exists an infinite number of possible eigenvectors corresponding to the different rotations in the space, and the JPSD is in general ill-defined. The condition $h(\lambda_1, \omega) = h(\lambda_2, \omega)$ when $\lambda_1 = \lambda_2$ deals with this ambiguity, as it ensures that the JPSD is the same independently of the choice of eigenvectors.
On the contrary, with JWSS-alternate one should construct an arbitrary basis of each eigenspace with multiplicity and set\footnote{More generally, in an analogy to~\cite{girault2015stationary} the JPSD could be block diagonal with each block being of size equal to the multiplicity.} $h(\lambda_1, \omega) \neq h(\lambda_2, \omega)$. This approach, which was followed in~\cite{segarra2018statistical}, features more degrees of freedom at the expense of the loss of filtering interpretation and higher computational complexity: one may not anymore use filters to estimate the JPSD (without reverting to Definition~\ref{def:time-vertex-stationarity}), whereas using the JFT to diagonalize the covariance scales like $\O(N^3 + N^2 T + N T \log(T))$. On the contrary, in our setting the PSD estimation complexity can be reduced to be close to linear in the number of edges $E$ and timesteps $T$ (see Appendix~\ref{app:fastpsd}). 
Nevertheless, we should mention that from a pragmatic perspective, the differences mentioned above are mostly academic. Eigenvalue multiplicities occur mainly when graph automorphisms exist. In the absence of such symmetries (e.g., in the graphs used in our experiments), the two definitions yield the same outcome.

\paragraph{Multivariate definitions.} On the other hand, joint stationarity can itself be derived as the combination of two multivariate versions of time/vertex stationarity, which we refer to respectively as MTWSS (see~\cite{bach2004learning}) and MVWSS. Before formally defining them in Definitions~\ref{def:time-stationarity} and~\ref{def:vertex-stationarity}, let us state our result formally: 

\begin{theorem} \label{theo:time-vertex-def}
    A joint process $\X$ is JWSS if and only if it is MTWSS and MVWSS.
\end{theorem}

To put this in context, we examine the two multivariate definitions independently.

\vspace{2mm}
(a) \emph{JWSS $\subset$ MTWSS.} The covariance matrix of a JWSS process has a block circulant structure, as $\bSigma_{t_1,t_2} =  \bSigma_{\delta,1} = \bGamma_\delta$, where $\delta = t_1-t_2 + 1$. Hence $\bSigma$ can be written as
\[
\bSigma_{\b{x}} = \left(\begin{array}{cccc}
\bGamma_{1}& \bGamma_{2} & \cdots & \bGamma_{T}\\
\bGamma_{T} & \bGamma_{1} & & \bGamma_{T-1} \\
\vdots &  &  \ddots & \vdots \\
\bGamma_{2} & \bGamma_{3}   &  \cdots & \bGamma_{1}
\end{array}\right),
\]
implying that correlations only depend on $\delta$ and not on any time localization. This property is shared by multivariate time wide-sense stationary processes:
\begin{definition}[MTWSS~\cite{bach2004learning}] \label{def:time-stationarity}
 A joint process $\X= \left[ \x_1, \x_2, \ldots, \x_T \right] \in \mathbb{R}^{N\times T}$ is Multivariate Time Wide-Sense Stationary (MTWSS), if and only if the following two properties hold:
\begin{enumerate}[label=(\alph*)]
    \item The expected value is constant as $\E{\x_t} = c\1$ for all $t$. 
    \item For all $t_1,t_2$ the second moment satisfies
    $
    \bSigma_{t_1,t_2} =  \bSigma_{\delta,1} = \bGamma_{\delta},  
    $
    where $\delta = t_1-t_2+1$.
\end{enumerate}
\end{definition}
Similarly to the univariate case, the Time Power Spectral Density (TPSD) is defined to encode the statistics of the process in the spectral domain
\begin{equation}
\hat{\bGamma}_\tau = \sum_{\delta=1}^{T}  \bGamma_{\delta} e^{-j\omega_\tau \delta}. 
\end{equation} 
We then obtain the TPSD of a JWSS process by constructing a graph filter from $h$ while fixing $\omega$. Setting $h_{\omega_\tau}(\lambda) = h(\lambda,\omega_\tau)$, the TPSD of a JWSS process is 
\shorten{\begin{equation}
    \hat{\bGamma}_\tau = h_{\omega_\tau}(\LG).
\end{equation}}
{$\hat{\bGamma}_\tau = h_{\omega_\tau}(\LG).$}

\vspace{2mm}
(b) \emph{JWSS $\subset$ MVWSS.}
For a JWSS process, each block  of $\bSigma$ has to be a linear graph filter, i.e., $\bSigma_{t_1,t_2}= \gamma_{t_1,t_2}(\LG)$, meaning that
\[
\bSigma = \left(\begin{array}{cccc}
\gamma_{1,1}(\LG)& \gamma_{1,2}(\LG) & \cdots & \gamma_{1,T}(\LG)\\
\gamma_{2,1}(\LG) & \gamma_{2,2}(\LG) & &  \\
\vdots &  &  \ddots & \vdots \\
\gamma_{T,1}(\LG) &   &  \cdots & \gamma_{T,T}(\LG)
\end{array}\right).
\]
This is perhaps better understood when compared to the multivariate version of vertex stationarity defined below: 
\begin{definition}[MVWSS] \label{def:vertex-stationarity}
A joint process $\X = [\x_1, \x_2, \ldots, \x_T] \in \mathbb{R}^{N\times T}$ is called Multivariate Vertex Wide-Sense Stationary
(MVWSS), if and only if the following two properties hold independently:
\begin{enumerate}[label=(\alph*)]
    \item The expected value is of each signal $\x_{t}$ is constant $\E{\x_{t}} = c_t\1$ for all $t.$
    \item For all $t_1$ and $t_2$, there exist a kernel $\gamma_{t_1,t_2}$ such that $ \bSigma_{t_1,t_2} =  \gamma_{t_1,t_2}(\LG) $.
\end{enumerate}
\end{definition}
It can be seen that every JWSS process must also be MVWSS, or equivalently, JWSS $\subset$ MVWSS.

\section{Joint Power Spectral Density estimation}
\label{sec:psd}

The joint stationarity assumption can be useful in overcoming the challenges associated with dimensionality. The main reason is that for JWSS processes the estimation variance is decoupled from the problem size. 
Concretely, suppose that we want to estimate the covariance matrix $\bSigma$ of a joint process $\x = \vec{\X}$ from $K$ samples $\x_{(1)}, \x_{(2)}, \ldots, \x_{(K)}$. As we show in the following, if the process is JWSS such that $\bSigma = h(\LG, \LT)$, the JPSD estimation variance is $O(1)$. This is a sharp decrease from the classical and MTWSS settings, for which $K \approx NT$ and $K \approx N$ realizations are necessary\footnote{The number of realizations needed for obtaining a good sample covariance matrix of an $n$-dimensional process is $O(n \log{n})$~\cite{rudelson1998,vershynin2012}.}, respectively. 

This section presents two JPSD estimators. The first provides unbiased estimates at a complexity that is $O(N^3 T \log(T))$. The second estimator decreases further the estimation variance at the cost of a bounded bias and is approximated with (close to) linear complexity. %

\subsection{Sample JPSD estimator} 

We define the \emph{sample JPSD estimator} for every graph frequency $\lambda_n$ and angular frequency $\omega_\tau$ as the estimate
\begin{align}
	\dot{h}(\lambda_n, \omega_\tau) \delequal \sum_{k = 1}^K \frac{\left| \JFT{\X_{(k)} }[n,\tau]\right|^2}{K}.
	\label{eq:PSD_sample}
\end{align}
In case the process does not have zero mean, it should be centered by subtracting the constant signal $c \, \1_N \1_T^\hermitian$, where 
$ c = \sum_{k,i,t} \X_{(k)}{[i,t]} / (KNT).$ In that case, the unbiased estimator should involve division by $K-1$, instead of $K$ as we have in~\eqref{eq:PSD_sample}.

\paragraph{Analysis.} For simplicity, in the following we suppose that the process is correctly centered. 
As the Theorem~\ref{theo:sample_JPSD} claims, the sample JPSD estimator is unbiased, and its variance decreases linearly with the number of samples $K$. 
\begin{theorem}
	For every distribution with bounded second and fourth order moments, the sample JPSD estimator $\dot{h}(\theta)$ 
	\begin{enumerate}[label=(\alph*)]
	\setlength{\itemindent}{-0mm}
		\item is unbiased, i.e., $\E{\dot{h}(\theta)} = h(\theta)$, and
		\item has variance $\var{ \dot{h}(\theta)} = h^2(\theta)\, \dfrac{ \gamma - 1}{K}$,
	\end{enumerate}	 
	where constant $\gamma$ depends only on the distribution of $\x$. %
	\label{theo:sample_JPSD}
\end{theorem}

\begin{proof}
For any $\theta = [\lambda, \omega]$,
the sample estimate is 
\begin{align}
    \dot{h}(\theta) = {h(\theta)} \sum_{k = 1}^K \frac{\hat{\varepsilon}_{(k)} \hat{\varepsilon}_{(k)}^\hermitian }{K},
    \label{eq:sample_estimate}
\end{align}
with $\hat{\varepsilon}_{(k)}$ being independent realizations of $\hat{\varepsilon}$, a zero mean complex random variable with unit variance. To see this, write $\x = h(\LG, \LT)^{1/2} \eps$, where the random vector $\eps$ has zero mean and identity covariance. Then, the complex random variable $\hat{\varepsilon}$ is the JFT coefficient of $\eps$ corresponding to frequencies $\lambda$ and $\omega$.
The bias follows by noting that $\E{\hat{\eps}_{(k)} \hat{\eps}_{(k)}^\hermitian} = 1$, for every $k$. The variance is computed similarly by exploiting the fact that different terms in the sum are independent as they correspond to distinct realizations and setting $\gamma = \E{|\hat{\varepsilon}|^4}$. 
\end{proof}

For the standard case of a Gaussian joint process, we provide an exact characterization of the distribution.
\begin{corollary}
	For every Gaussian JWSS process, the sample JPSD estimate follows a Gamma distribution with shape $ K/2$ and scale $2h(\theta)/K$. The estimation error variance is equal to $\var{ \dot{h}(\theta)} = 2\, h^2(\theta)/K$.
	\label{cor:sample_estimator_gaussian}
\end{corollary}
\begin{proof}
We continue in the context of the proof of Theorem~\ref{theo:sample_JPSD}. For a Gaussian distribution, $\hat{\varepsilon}$ is centered and scaled Gaussian and thus $\hat{\varepsilon}^2$ is a chi-squared random variable with 1 degree of freedom. Our estimate is, therefore, a scaled sum of i.i.d. chi-squared variables and corresponds to a Gamma distribution. The corollary then follows directly.
\end{proof}

Observe that the variance depends linearly on the fourth-order moment of $|\hat{\varepsilon}|$ (see proof of Theorem~\ref{theo:sample_JPSD}) and is inversely proportional to the number of samples, but it is independent of $N$ and $T$. This implies that $\norm{\bSigma - \dot{\bSigma}}_2$ can be made arbitrarily small using $K = O(1)$ samples. In the following, we discuss how to achieve an even smaller variance by exploiting the properties of $h(\theta)$. 

\subsection{Convolutional JPSD estimator}

When the number of available realizations $K$ is small (even 1), one may make use of additional assumptions on to obtain reasonable estimates. To this end, we next present a parametric JPSD estimator that allows us to trade off bias for variance.

Before delving into JWSS processes, it is helpful to consider the purely temporal case.
For a TWSS process, it is customary to assume that the autocorrelation function has support $L$ that is a few times smaller than $T$. Then, cutting the signal into $\frac{T}{L}$ smaller parts and computing the average estimate, reduces the variance (by a factor of $\frac{T}{L}$), without sacrificing frequency resolution. 
This basic idea stems from two established methods used to estimate the PSD of a temporal signal, namely Bartlett's and Welch's method~\cite{bartlett1950periodogram,welch1967use}. 
Averaging across different windows is equivalent to smoothing the TPSD by convolving it with a window in the frequency domain: this results in attenuation of the correlation for long delays, enforcing localization in the time domain.

\paragraph{Estimator.} Armed with this interpretation, we proceed by smoothing the JPSD with a user-specified bi-variate window $g$, such as a Gaussian or a disc window. The convolutional JPSD estimator computes the JPSD at joint frequency $\theta = (\lambda, \omega)$ as:
\begin{align}
	 \ddot{h}(\theta) 
	& \delequal \frac{1}{ c_g(\theta)} \sum_{\substack{n=1\\\tau = 1}}^{N, T} g(\theta - \theta_{n,\tau})^2 \, \sum_{k = 1}^K \frac{\left| \JFT{\X_{(k)} }[n,\tau]\right|^2}{K},  \label{eq:JPSD-est}
\end{align}
where $c_g(\theta) \delequal \sum_{n,\tau} g(\theta - \theta_{n,\tau})^2$ is a normalization factor. %
For implementation specifics, including a discussion on the choice of the bivariate kernel $g$, we refer the reader to Appendix~\ref{app:fastpsd}.

The convolutional JPSD estimator is related to known PSD estimators for TWSS and VWSS processes. Denote by $\phi$ the Dirac function. We have that: (a) For $g(\theta) = \phi(\lambda) \cdot g_T(\omega)$, we recover the classical TPSD estimator, applied independently for each $\lambda$. (b) For $g(\theta) = g_G(\lambda) \cdot \phi(\omega)$, we recover the VPSD estimator from~\cite{perraudin2016stationary} applied independently for each $\omega$. 
Similar to the latter, the estimator can be closely approximated at a complexity that is linear w.r.t. the number of graph edges/nodes, and up to a logarithmic factor linear to the number of timesteps (see Appendix~\ref{app:fastpsd}).    

\paragraph{Analysis.} To provide a meaningful bias analysis, we introduce a Lipschitz continuity assumption on the JPSD, matching the intuition that localized phenomena tend to have a smooth representation in the frequency domain.   

\begin{theorem}\label{theo:convolutional_bias_variance}
At $\theta,$ the convolutional JPSD estimator $\ddot{h}(\theta)$  
\begin{enumerate}[label=(\alph*)]
\setlength{\itemindent}{0mm}
	\item has bias 
	\begin{align*}\hspace{-6mm}
\left|\E{\ddot{h}(\theta) - h(\theta)} \right|
\leq \frac{\epsilon }{ c_g(\theta) }\hspace{-1mm} \sum_{n=1,\tau=1}^{T,N}\hspace{-3mm}g(\theta - \theta_{n,\tau})^{2} \norm{\theta - \theta_{n,\tau}}_{2},
\end{align*}
	where $\epsilon$ is the Lipschitz constant of $h(\theta)$, and
	\item when the entries of $\hat{\X}$ are independent random variables, its variance is 
	\begin{align*}
		\var{\ddot{h}(\theta)} = \sum_{n,\tau} \frac{g(\theta- \theta_{n,\tau})^{4}}{c_g(\theta)^2} \, \var{\dot{h}(\theta_{n,\tau})},
	\end{align*}
	where $\var{\dot{h}(\theta_{n,\tau})}$ is the variance of the sample JPSD estimator at $\theta_{n,\tau}$.
\end{enumerate}
\end{theorem}

The derivations of the bias and variance are given in Lemmas~\ref{lemma:convolutional_bias} and~\ref{lemma:convolutional_variance}, respectively.
\vspace{2mm}

We note two corner cases of interest. In the most convenient case, the JPSD is constant, and our estimator is unbiased (the Lipschitz constant $\epsilon$ is zero). On the other hand, if the JPSD fluctuates rapidly, the bias of the estimate will be significant unless $g$ is close to a Dirac. Here, the sample estimator should be preferred.  

We further consider as a theoretical example the case of a Gaussian JWSS process and a (spectral) disc window with bandwidth B, i.e., $g_{B}(\theta) = 1$ if $\norm{\theta}_2 \leq\frac{B}{2}$ and 0, otherwise. Though perhaps not the most practical choice from a computational perspective, we consider here a disc window because it leads to simple and intuitive estimates.
\begin{corollary}
For every $\epsilon$-Lipschitz Gaussian JWSS process and disc window $g_B(\theta)$, the convolutional estimate has
\begin{align}
	\hspace{-3mm}\left|\E{\ddot{h}(\theta) - h(\theta)} \right| \leq \frac{\epsilon B}{2} \quad \text{and} \quad\var{\ddot{h}(\theta)} 
	&= \frac{2 \,{h^2_\mathcal{S}}}{K |\mathcal{S_\theta}|},
\end{align}
with set $\mathcal{S}_\theta = \{ \theta_{n,\tau} \, | \, \norm{\theta_{n,\tau} - \theta}_2 \leq B/2 \}$ and ${h^2_\mathcal{S}} = \sum_{ \theta_{n,\tau} \in \mathcal{S} } h(\theta_{n,\tau})^2$.
\end{corollary}
\begin{proof}
The results follow from Theorem~\ref{theo:convolutional_bias_variance} and Corollary~\ref{cor:sample_estimator_gaussian} by noting that when a disc window is used: (a) $c_g(\theta) = |\mathcal{S_\theta}|$, and (b) $g(\theta- \theta_{n,\tau})^{2} = 1$ for all $n,\tau$ in the window (there are $|\mathcal{S_\theta}|$ in total) and zero otherwise. The independence condition required by the variance clause of the theorem is satisfied since $\hat{\x}$ is Gaussian (as a rotation $\hat{\x} = \UJ^\hermitian \x$ of a Gaussian vector) with diagonal covariance. 
\end{proof}
The above result suggests that by selecting our window (bandwidth), we can trade off bias for variance. The trade-off is particularly beneficial as long as (a) the JPSD is smooth relatively to the disc size ($\epsilon B \ll 1$) and (b) the graph eigenvalues are clustered ($|S_\theta|\gg 1$ when $h(\theta)\gg 0$).

\section{Recovery of JWSS Processes}
\label{sec:recovery}

This section considers the MMSE problem of recovering a JWSS process $\x = \vec{\X}$ from linear measurements $\y$ corrupted by a zero-mean JWSS process $\w$: 
\begin{align}\label{prob:recovery_problem}
\begin{aligned}
& \min_{f: \Rbb^{N'} \rightarrow \Rbb^N}
& & \E{\| f(\y) - \x\|_2^2}  \\
& \text{subject to}
& &  \y = \A \x + \w, 
\end{aligned}
\tag{P0}
\end{align}
where the function $f$ is linear on $\y$, i.e., there exists a matrix $\b{W}$ and a vector $\b{b}$ such that $f(\y) = \b{W} \y+ \b{b}$.
We remark that (a) for $\A$ binary diagonal and $\w = \0$, \eqref{prob:recovery_problem} is an \emph{interpolation} problem, (b) for $\A = \I$ and $\w$ white noise (P0) is a \emph{denoising} problem, and (c) for $\A$ diagonal with $\A_{ii} = 1$ if $i \leq N t$ and zero otherwise and $\w = \0$ it corresponds to \emph{forecasting}. We mainly consider the former two problems since, for forecasting, it is more computationally efficient to utilize autoregressive models~\cite{loukas2016predicting}.

The \emph{minimum mean-squared linear estimate} is known to be
\begin{eqnarray}
\label{eq:linear_mmse}
\dot{\x} &= \bSigma_{\x\y} \bSigma_{\b{y}}^{-1} (\y - \bar{\y}) + \bar{\x},
\end{eqnarray}
with the definitions $\bSigma_{\y} = \A \bSigma \A^\hermitian + \bSigma_\w$ and $\bSigma_{\x\y}  = \bSigma \A^\hermitian $. Obtaining $\dot{\x}$ therefore entails solving a linear system in matrix $\bSigma_{\y}$, that -naively approached- has $O( N^2 T^2)$ complexity. In addition, the condition number of $\bSigma_{\y}$ can be large, rendering direct inversion unstable. For instance, this may happen when one attempts to reverse any smoothing operation $\A$ that severely attenuates part of the signal's spectrum.   

\vspace{3mm} \noindent We next discuss how to deal with these issues: 

\paragraph{Decreasing the complexity.} Thankfully, even if $\bSigma_\y$ is not always sparse, we can approximate its multiplication by a vector without actually computing it as (a) $\A$ is, for many applications (denoising, prediction, forecasting), sparse, and (b) per our assumption $\bSigma$ and $\bSigma_\w$ are joint filters and therefore they can be implemented at complexity that is (up to logarithmic factors) linear to the number of edges $E$ and timesteps $T$~\cite{isufi2017autoregressive,isufi2016separable,grassi2017timevertex}. Therefore, if we employ an iterative method such as the (preconditioned) conjugate gradient to compute the solution, the complexity of each iteration will be linear on the problem size. 

\paragraph{Singular or badly conditioned $\bSigma_{\y}$.} We choose the solution with the minimal residual by substituting the inverse $\bSigma_{\b{y}}^{-1}$ in~\eqref{eq:linear_mmse} with the pseudo-inverse $\bSigma_{\b{y}}^{+}$.
However, instead of solving the normal equations 
$
\dot{\x} = \bSigma_{\x\y} (\bSigma_{\b{y}}^2)^{-1} \bSigma_{\b{y}} (\y - \bar{\y})+ \bar{\x},    
$
which has the effect of significantly increasing the condition number of our matrix, we suggest to employ the minimal residual conjugate gradient method for symmetric matrices~\cite{axelsson1980conjugate}. %
For badly conditioned covariance matrices, an alternative solution is to rewrite the problem as a regularized least squares problem
\begin{align}\label{prob:Wiener-opt}
\min_{\b{z} \in \Rbb^N} \|\A\b{z}- \y\|_2^2 + \|{h_\w(\LG,\LT)}^{\sfrac{1}{2}} \, h_\x(\LG,\LT)^{-\sfrac{1}{2}} (\b{z} -\bar{\x}) \|_2^2
\end{align}
and solve it using the generalization of the fast iterative shrinkage-thresholding algorithm (FISTA) scheme~\cite{combettes2005signal,combettes2011proximal,komodakis2015playing}. This problem was shown to converge to the correct solution when $\w$ is white noise. 
More details about the optimization procedures can be found in~\cite{perraudin2016stationary}.
Similarly, in the noiseless case one removes term $\|\A\b{z}- \y\|_2^2$ in \eqref{prob:Wiener-opt} and introduces instead the constraint $\A \b{z}= \y$. The resulting optimization problem can be solved  
using a Douglas-Rachford scheme~\cite{combettes2007douglas}.

\begin{figure}[t!]
\centering
\includegraphics[width=0.24\columnwidth]{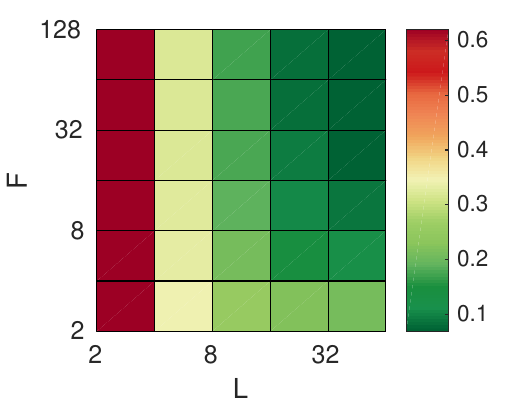}%
\includegraphics[width=0.24\columnwidth]{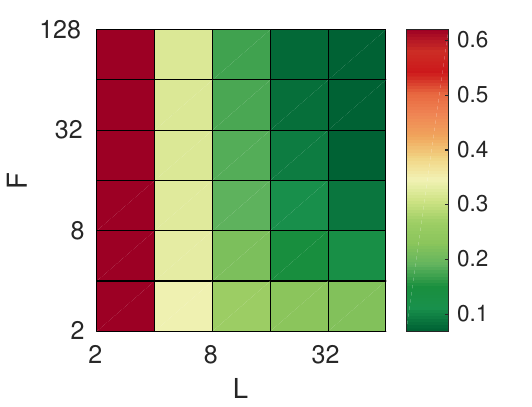}%
\includegraphics[width=0.24\columnwidth]{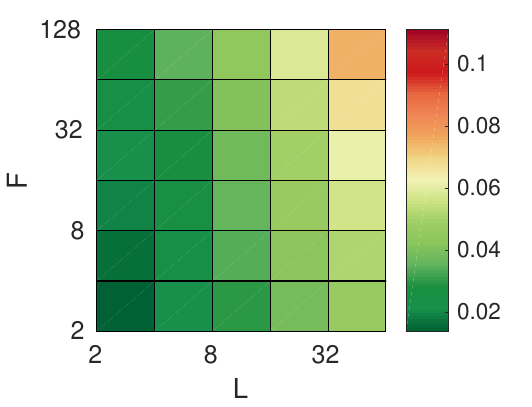}%
\includegraphics[width=0.24\columnwidth]{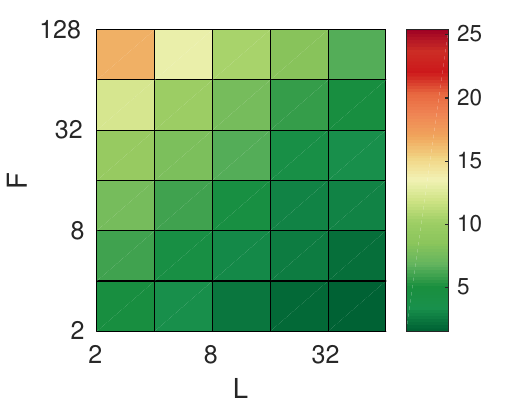}\\
\hspace{0.1cm}\text{\scriptsize{(a) estimation error}} 
\hspace{2.4cm}\text{\scriptsize{(b) bias}}
\hspace{2.5cm}\text{\scriptsize{(c) standard deviation}}%
\hspace{1.5cm}\text{\scriptsize{(d) execution time (sec)}}\\
\vspace{2mm}
\caption{Influence of the parameters (window size $L$ and number of graph filters $F$) on the (a) estimation error, (b) bias, (c) normalized standard deviation, and (d) execution time. For improved visibility, the scale of (c) has been changed. \label{fig:JPSD_accuracy}}%
\end{figure}
\section{Experiments}
\label{sec:experiments}
\subsection{Joint Power Spectral Density Estimation}

The first step in our evaluation is to analyze the efficiency of JPSD estimation. Our objective is dual. First, we aim to study the role of the different method parameters into the estimation accuracy and computational complexity, essentially providing practical guidelines for their usage. In addition, we wish to illustrate the usefulness of the joint stationarity assumption, even when the graph is only approximately known. 
 
\paragraph{Variance-bias-complexity tradeoffs.} To validate the analysis of Section~\ref{sec:psd} for the computational and accuracy trade-offs inherent to our JPSD estimation method, we performed numerical experiments with random geometric graphs of $N = 256$ vertices (we build a 10-nearest neighbor graph, weighted by a radial basis function kernel tuned so that the average weighted degree is slightly above $7$) and JWSS processes ($T = 128$ timesteps).
Though our approach works with any JPSD, including high frequency ones, in this experiment we consider a stochastic process generated by the discrete damped wave equation with a non-separable JPSD $h(\lambda,\omega) = \exp(-|\omega|/2) \cos(\omega \, \text{acos}(1-\lambda))$

\emph{Variance-bias.} First, we examine the relation between the real JPSD $h$ and the convolutional estimate $\ddot{h}$ obtained using the `fast' method described in Appendix~\ref{app:fastpsd}. We use the following metrics: 
\begin{align*}
\scalebox{0.95}{%
\begin{tabular}{c | c | c} 
\textit{error} & \textit{bias} & \textit{standard deviation} \\[0.3em]
$\dfrac{\Es{\norm{\ddot{\H} - \H}_F} }{ \norm{\H}_F} $ & $\dfrac{\norm{\Es{\ddot{\H}} - {\H}}_F}{\norm{\H}_F}$ & $\dfrac{\Es{\norm{ \ddot{\H} - \Es{\ddot{\H}}}_F}}{\norm{\H}_F}$,
\end{tabular}
}
\end{align*}
where $\H = h(\bLambda_G, \bOmega)$, $\ddot{\H} = \ddot{h}(\bLambda_G, \bOmega)$ and $\Es{\cdot}$ is the empirical expectation computed over $20$ independent experiments.
We remind the reader that there are two parameters influencing the performance of the convolutional JPSD estimator (see Appendix~\ref{app:fastpsd}): the window size $L$ corresponding to our assumption for the support length of the autocorrelation in time, and the number of graph filters $F$ used to capture power density in the graph spectral dimension. As discussed in Theorem~\ref{theo:convolutional_bias_variance}, the bias will be small as long as the JPSD is a smooth function (it has a small Lipschitz constant $\epsilon$), in which case one may opt for small $L$ and $F$. Figures~\ref{fig:JPSD_accuracy} (a-d) report four key metrics for an exhaustive search of $L,F$ combinations. We observe that large values of $F$ and $L$ generally reduce the estimation error (Figure~\ref{fig:JPSD_accuracy} (a)) because they result in reduced bias (Figure~\ref{fig:JPSD_accuracy} (b)). Nevertheless,   
setting the parameters to their maximum values is not suggested as the variance is increased (Figure~\ref{fig:JPSD_accuracy} (c)).

\begin{figure}
\centering
\includegraphics[width=0.75\columnwidth]{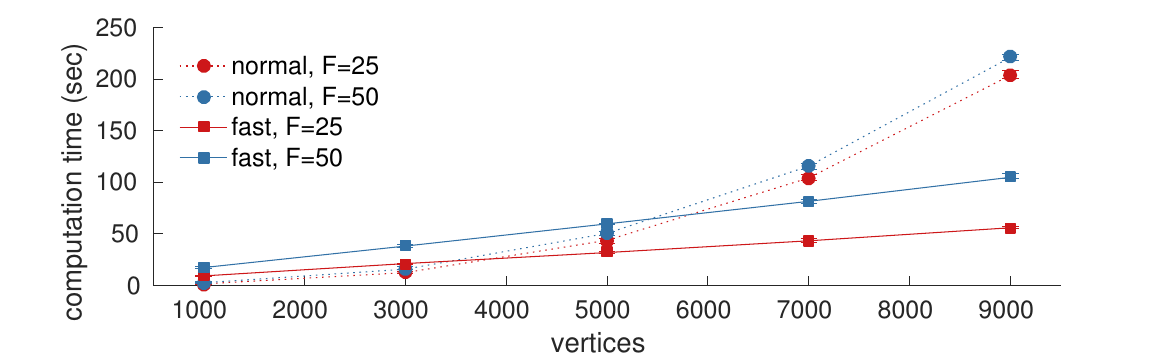}
\caption{Scalability of the convolutional JPSD estimator in seconds (vertical axis) w.r.t. the number of vertices (horizontal axis). The fast implementation should be favored when the graph is composed of more than a few thousand vertices. The approximation error of the fast implementation was negligible in our experiments.}%
\label{fig:scalability}
\end{figure}

\emph{Complexity.} In Figure~\ref{fig:JPSD_accuracy} (d) we see that utilizing a large number of filters (i.e., large $F$) increases the average execution time. Figure~\ref{fig:scalability} delves further into the issue of scalability. In particular, we vary the number of vertices from 1000 to 9000 and focus on a process with JPSD $h(\theta) = e^{-\lambda/\lambda_\textit{max}} \, e^{-5\, \omega^2}$. We then examine the min/median/max execution time of the convolutional JPSD estimator for a for increasing problem sizes when ran in a desktop computer and repeated 10 times. We compare two implementations. The first, which naively performs the convolution in the spectral domain, uses the eigenvalue decomposition and therefore scales quadratically with the number of vertices. Due to its optimized code and simplicity, this should be the method of choice when $N$ is small. For larger problems, we suggest using the fast implementation. As shown in the figure, this scales linearly with $N$ (here $E = \O(N)$) when the number of filters $F$ and timesteps $T$ are held constant. In this experiment, we set $L$ to 64. %

\emph{How to choose $L$ and $F$?} Having no computational constrains, one should choose the parameter combination that minimized the Akaike information criterion (AIC) score $\text{AIC} = 2FL - 2\ln(\ddot{\ell})$, where $\ddot{\ell}$ is the distribution dependent estimated likelihood $\ddot{\ell} = \mathbf{P}(\x | \ddot{\bSigma})$, and $\ddot{\bSigma}$ is the estimated covariance based on the convolutional JPSD estimator with parameters $L$ and $F$~\cite{akaike1974new}. This procedure is often unfeasible as it is based on computing each model's log-likelihood and thus entails estimating one JPSD for each parameterization in consideration (as well as knowing the distribution type). We have found experimentally that setting $F=\text{min}(N, 50)$ provides a good trade-off between computational complexity and error. On the other hand, we suggest setting $L$ to an upper bound of the autocorrelation support.

\paragraph{Learning from few realizations and a noisy graph.} Figure~\ref{fig:covariance_estimation_comparison}  illustrates the benefit of a joint stationarity prior as compared to (a) an empirical covariance estimator which makes no assumptions about the data, and (b) the MTWSS process estimator with optimal bandwidth~\cite{wiener1957prediction}. As expected, accurate estimation is challenging when the number of realizations is much smaller than the number of problem variables ($N T$), returning errors above one for the empirical estimator. Introducing stationarity priors regularizes the estimation resulting in more stable estimates. 

\begin{figure}
\centering
\includegraphics[width=0.75\columnwidth]{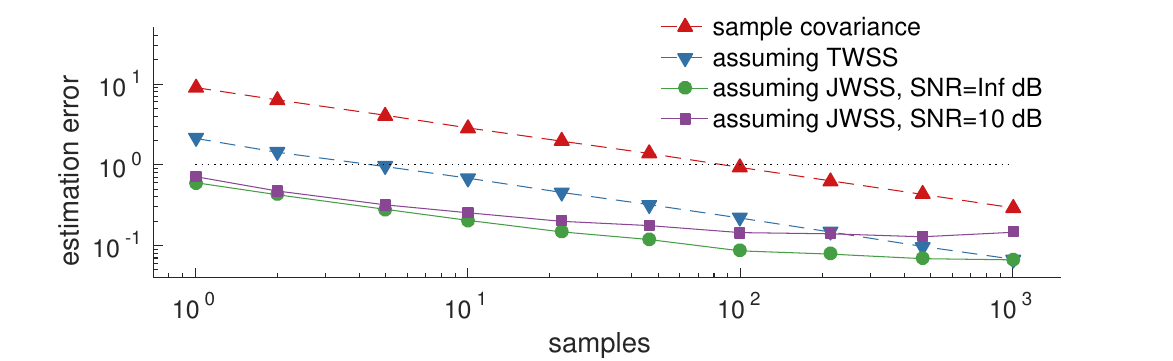}\\
\text{(a)$N = 10,\ T = 10$}\\
\includegraphics[width=0.75\columnwidth]{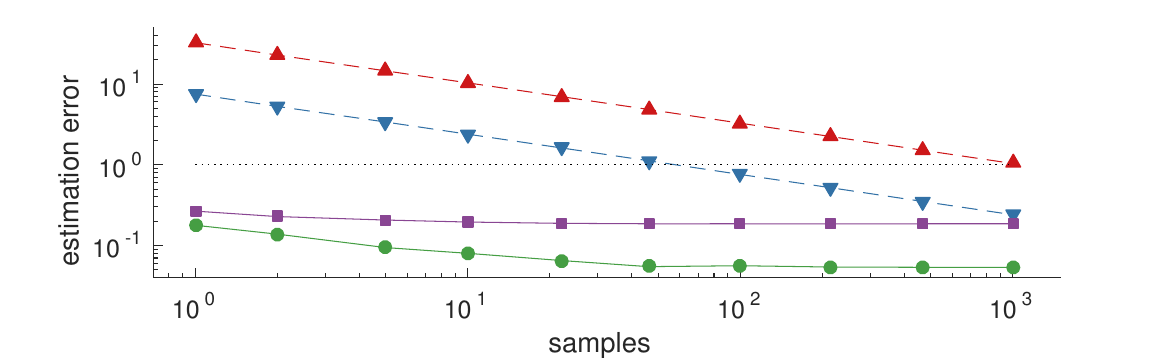}\\
\text{(b) $N = 100,\ T = 10$}\\
\vspace{2mm}

\caption{Estimation error ${\Es{\norm{\ddot{\H} - \H}_F} }/{ \norm{\H}_F}$ as a function of the number of realizations and number of vertices. Even an approximate knowledge of the graph enables us to make good estimates of the covariance (and PSD) from few realizations. The joint stationarity prior becomes especially meaningful when the number of variables ($N,T$) increases. The benefit also holds for a noisy graph (SNR = 10dB). }
\label{fig:covariance_estimation_comparison}
\end{figure}

What is perhaps surprising is that even when the graph (and $\UG$) are known only approximately,  estimating the second order moment of the distribution using the joint stationarity assumption is beneficial. To portray this phenomenon, we also plot the estimation error when using a noisy graph (we corrupted the weighted adjacency matrix by Gaussian noise, resulting in an SNR of 10 dB). Undoubtedly, introducing noise to the graph edges negatively affects estimation by introducing bias. Still, even with noise, the proposed method significantly outperforms purely time-based methods when less than $NT$ realizations are available.  

\subsection{Recovery Performance on Three Datasets}

We apply our methods on three diverse datasets featuring multivariate processes evolving over graphs: (a) a weather dataset depicting the temperature of 32 weather stations over one month, (b) a traffic dataset depicting high resolution daily vehicle flow of 4 weekdays, and (c) SIRS-type epidemics in Europe. Our experiments aim to show that joint stationarity is a useful model, even in datasets which may violate the strict conditions of our definition, and that it can yield a significant improvement in recovery performance, as compared to time- or vertex-based stationarity methods.

\paragraph{Experimental setup}. We split the $K$ realizations of each dataset into a \emph{training set} of size $p_t K$ and a \emph{test set} of size $ (1 - p_t) K$, respectively.  The training set is used to estimate the JPSD. Then, in the first two experiments, we attempt to recover the values of $p_d NT$ variables randomly discarded from the test set. This corresponds to $\A$ being a binary diagonal matrix and $\w = 0$ in Problem~\ref{prob:recovery_problem}, for which the solution is not given by a Wiener filter. In the third experiment, we instead consider a denoising problem with $\A = \I$ and $\w$ being a random Gaussian vector. In each case, we report the RMSE for the recovered signal normalized by the $\ell_2$-norm of the original signal. 
We compare our joint method with the sample and convolutional JPSD estimators to \emph{univariate time/vertex stationarity}~\cite{perraudin2016stationary}. These methods solve the statistical recovery problem under the assumption that signals at stationary in the time/vertex domains, but considering different vertices/timesteps as independent. These methods are known to outperform non-model based methods, such as Tikhonov regularization (ridge regression) and total-variation regularization (lasso) over the time or graph dimensions~\cite{shuman2013emerging,sandryhaila2013discrete}. We also compare to the more involved \emph{MTWSS model}~\cite{bach2004learning} where the values at different vertices are correlated and the covariance is block-circulant of size $NT \times NT$ (see Definition~\ref{def:time-stationarity}). The latter is only shown for the weather dataset as the large number of variables present in the other datasets (e.g., $\approx 10^8$ parameters for the traffic dataset) prohibited computation. We remark that the graph Laplacians we considered did {not} possess eigenvalue multiplicities, meaning that the results obtained using the JWSS-alternate definition are identical to that with JWSS using a sample JPSD estimator---thus, we do not include JWSS-alternate in our comparison. 

\begin{figure}[t!]
\centering
\includegraphics[width=0.75\columnwidth]{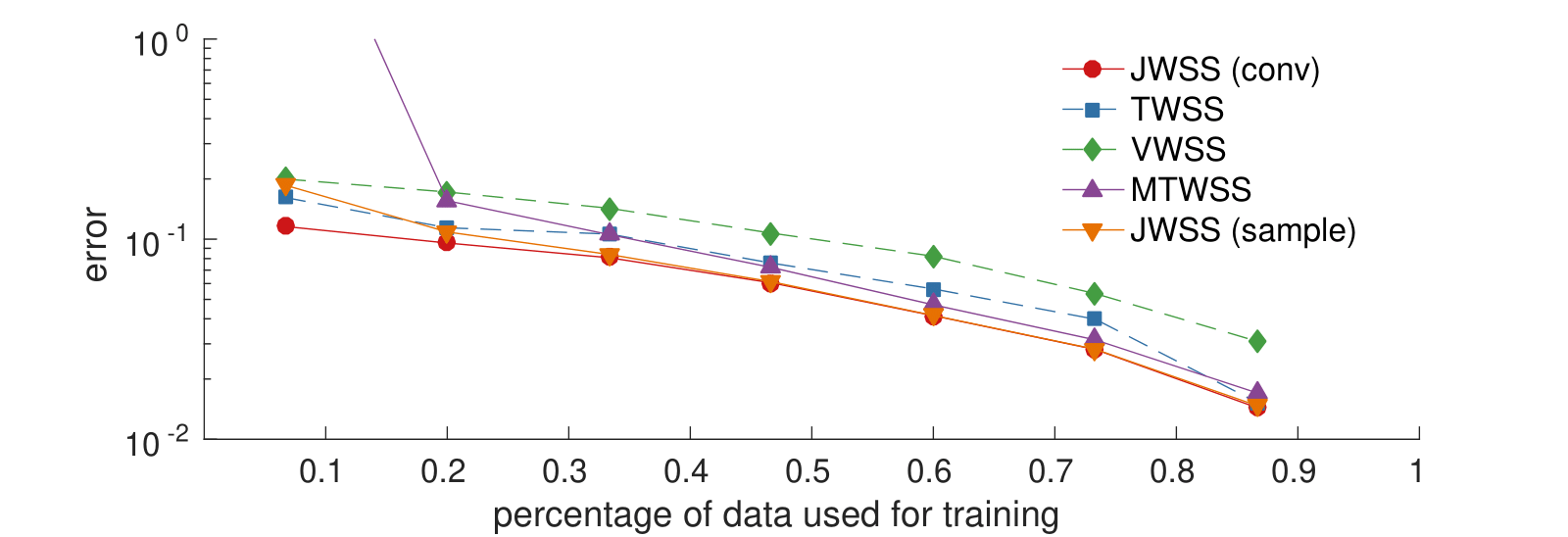} \\
\text{(a) Influence of the training set size ($p_d = 30\%$)}\\
\includegraphics[width=0.75\columnwidth]{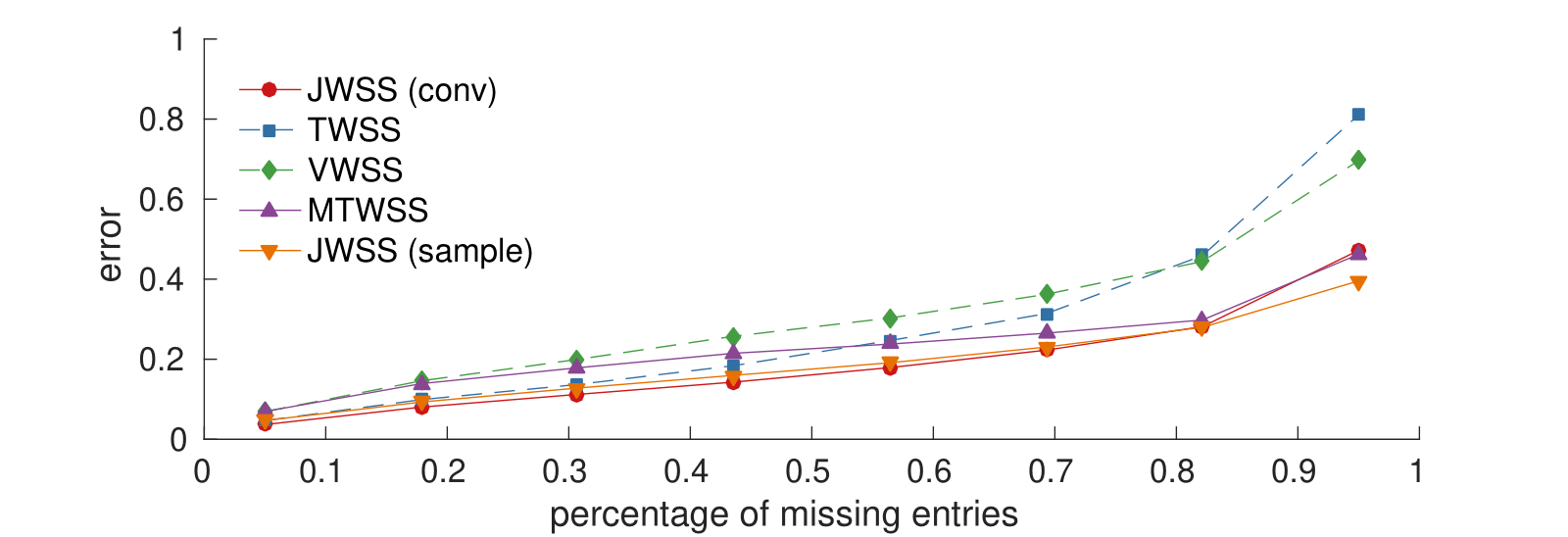}\\
\text{(b) Influence of the percentage of missing values ($p_t = 20\%$)}\\
\vspace{2mm}
\caption{Experiments with weather data. The joint approach becomes especially meaningful when the available data are few.}
\label{fig:result_molene}
\end{figure}

\paragraph{Molene dataset.} The French national meteorological service has published in open access a dataset\footnote{Access to the raw data is possible directly from \url{https://donneespubliques.meteofrance.fr/donnees_libres/Hackathon/RADOMEH.tar.gz}} with hourly weather observations collected during the Month of January 2014 in the region of Brest (France)~\cite{girault2015stationary}. The graph was built from the coordinates of the weather stations by connecting all the neighbors in a given radius with a weight function $\W_G[i_1,i_2] = \mathrm{exp}({-k \, d(i_1,i_2)^2 })$, where $d(i_1,i_2)$ is the Euclidean distance between the stations $i_1$ and $i_2$. Parameter $k$ was adjusted to obtain an average degree around $5$ ($k$ however is not a sensitive parameter). We split the data in $K = 15$ consecutive periods of $T = 48$ hours each. As sole pre-processing, we removed the mean (over time and stations) of the temperature\footnote{Though computing separate means (one for each data chunk) yields slightly better performance, to be consistent with the proposed model we computed a single mean over all training data.}. 
Since $NT$ is here relatively small, we used the empirical JPSD estimator.

We first test the influence of training set size $p_t$, while discarding $p_d = 30\%$ of the test variables. As seen in Figure~\ref{fig:result_molene} (a), due to its large sample complexity the MTWSS approach provides good recovery estimates when the number of realizations is large, approaching that of joint stationarity, but suffers for small training sets (though not shown in the figure, the relative mean error was 9.8 when only $p_t = 10\%$ of the data was used for training). Due to their stricter modeling assumptions, univariate stationarity methods returned relevant estimates when trained from few realizations, but exhibited larger bias. 
The convolutional JPSD estimator can be seen to improve upon the sample estimator when the amount of data used for JPSD estimation is small (less than 20\%). For bigger training sets, the two estimators yield similar accuracy.
Figure~\ref{fig:result_molene} (b) reports the achieved errors for recovery problems with progressively larger percentage $5\% \leq p_d \leq 95\%$ of discarded entries for a training percentage of $p_t = 20\%$. We can observe that the error trends are consistent across all cases.

\begin{figure}[t!]
\centering
\includegraphics[width=0.75\columnwidth]{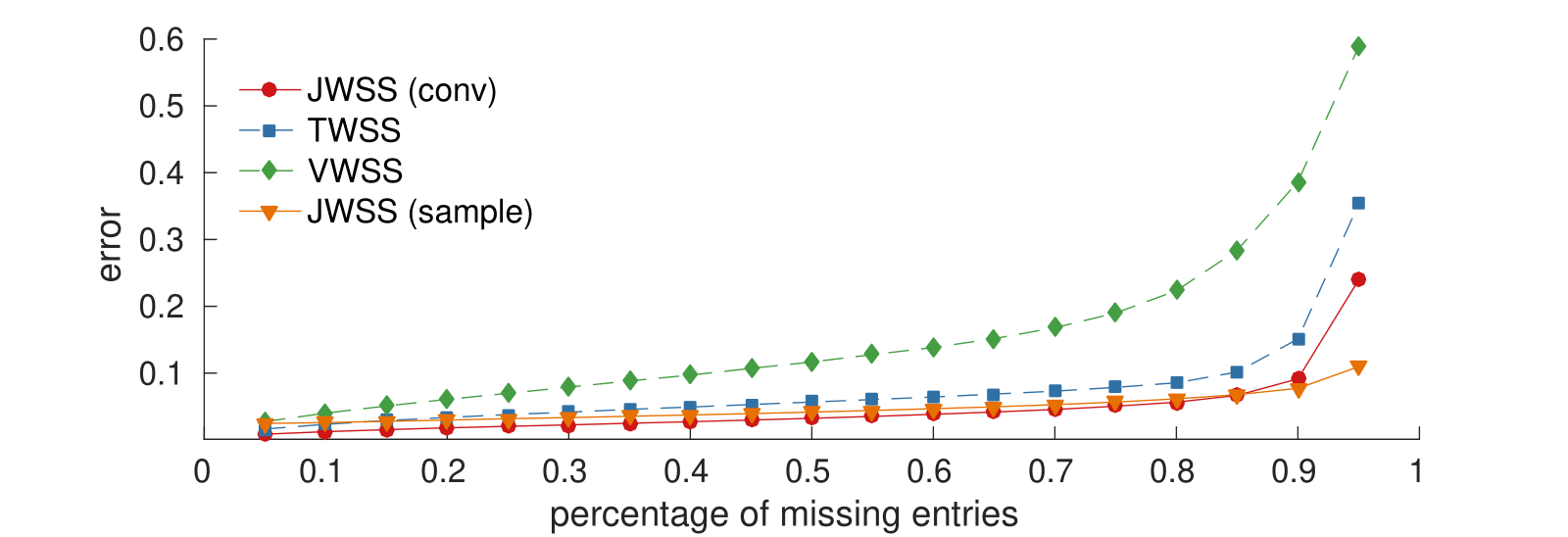}\\
\text{(a) 1 out of 4 days used for training ($p_t = 25\%$)}\\
\includegraphics[width=0.75\columnwidth]{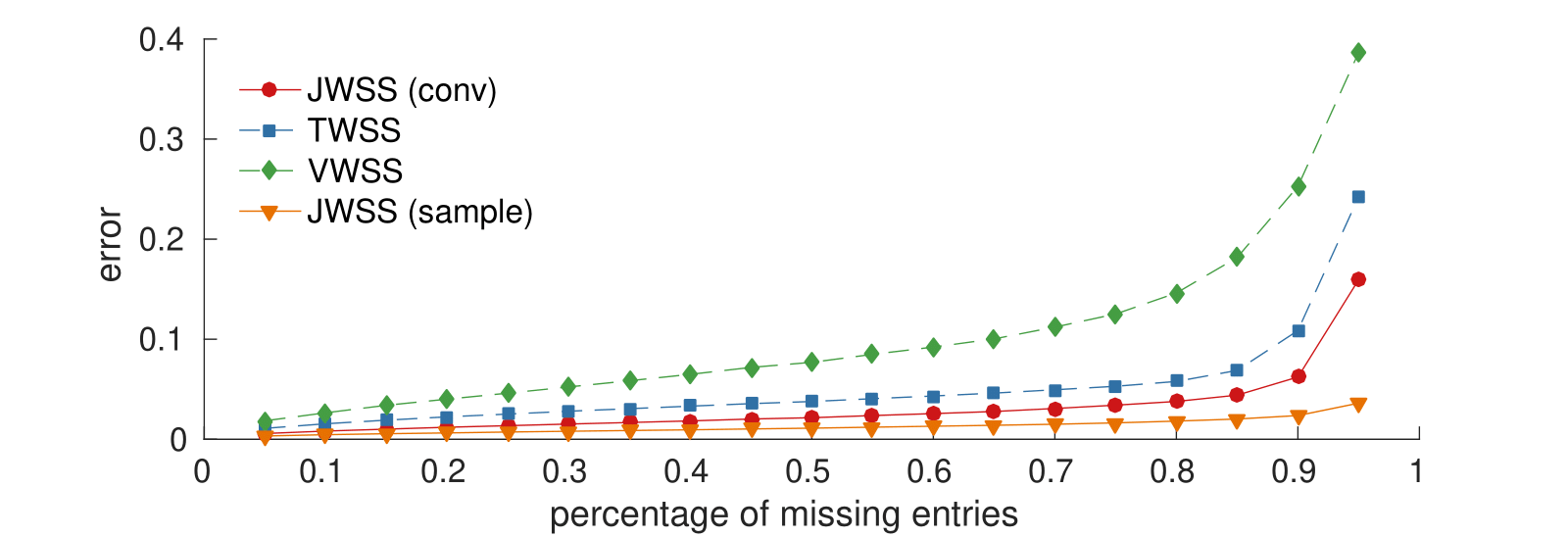}\\
\text{(b) 3 out of 4 days used for training ($p_t = 75\%$)}\\
\vspace{2mm}
\caption{Experiments on Sacramento highway flow. By exploiting both graph and temporal dimensions, the joint approach closely captures the subtle variations in traffic throughout each weekday.}
\label{fig:result_pems}
\end{figure}

\paragraph{Traffic dataset.} The California department of transportation publishes high-resolution traffic flow measurements (number of vehicles per unit interval) from stations deployed in the highways of Sacramento\footnote{The data correspond to the 3rd district of California and can be downloaded from \url{http://pems.dot.ca.gov/}}. We focused on 727 stations over four weekdays in the period 01-06 April 2016. Starting from the road connectivity network obtained by the OpenStreetMap.org, we constructed one time serie for each highway segment by setting the flow over it to be a weighted average of all nearby stations, while abiding to traffic direction. This resulted in a graph of $N = 710$ vertices, and a total of $T = 24 \times 12$ measurements per day for $K = 4$ days. We used the convolutional JPSD estimator with parameters $L = T/2$ and $F = 75$, which were experimentally found to give good performance in the training set.

Figures~\ref{fig:result_pems} (a) and~\ref{fig:result_pems} (b) depict the mean recovery errors when the training sets where 1 ($p_t = $ 25\%) and 3 days ($p_t = $ 75\%) respectively. The strong temporal correlations present in highway traffic were useful in recovering missing values. Considering both the temporal and spatial dimensions of the problem, resulted in accurate estimates, with less than 0.04 error when $p_d=$50\% of the data were removed and the PSD was estimated from one day. As expected, the convolutional estimator is efficient in the case when the training set is small (1 out of 4 days used for training): assuming that the JPSD is smooth helps to reduce estimation variance and computational complexity, but can lead to a slight decrease in accuracy when a large amount of training data is available.

\begin{figure}[t!]
\centering
\includegraphics[width=0.75\columnwidth]{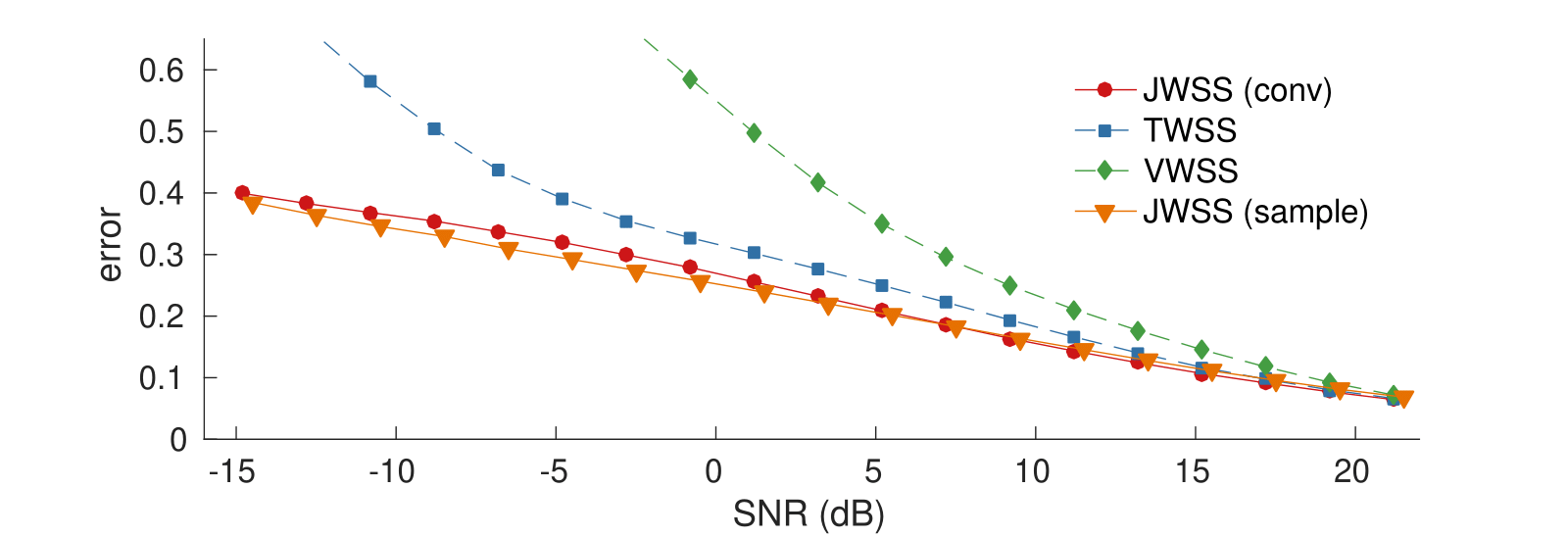}\\
\text{(a) Influence of noise level ($p_t = 50\%$)}\\
\includegraphics[width=0.75\columnwidth]{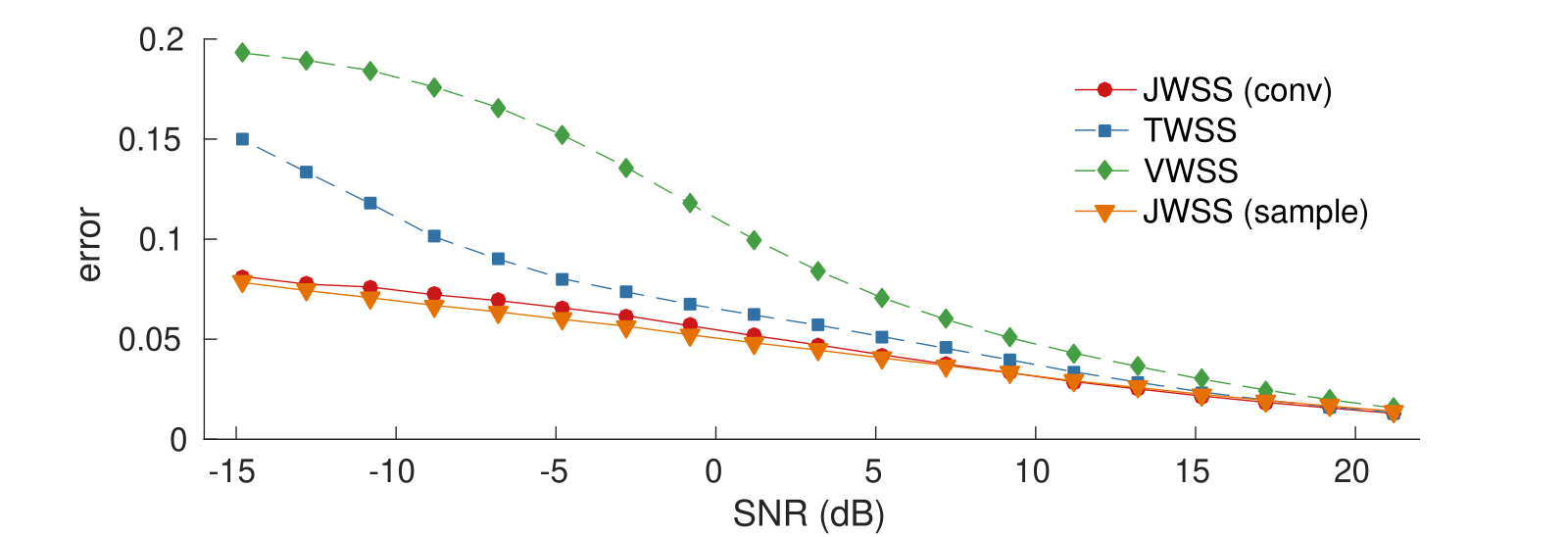}\\
\text{(b) Influence of noise level ($p_t = 90\%$)}\\
\caption{Experiments with the SIRS epidemic model.  }
\label{fig:result_sirs}
\end{figure}

\paragraph{SIRS epidemic.} Our third experiment simulates the spread of an infectious disease over $N = 200$ major cities of Europe, as predicted by the Susceptible-Infected-Recovered-Susceptible (SIRS) model, one of the standard models used to study epidemics. We intend to examine the predictive power of the considered methods when dealing with different realizations of a non-linear and probabilistic process over a graph (the data are fictitious). We parameterized SIRS as follows: length of infection period: 2 days, length of immunity period: 10 days, probability of contagion across neighboring cities per day: 0.005, total period: $T = 180$ days. We generated a total of $K =10$ infections, all having the same starting point. 

In contrast to the previous experiments, here we attempt to recover the data after they have been corrupted with additive Gaussian noise. Figures~\ref{fig:result_sirs} (a) and~\ref{fig:result_sirs} (b) depict the mean recovery error as a function of the input signal-to-noise ratio (SNR), respectively when $p_t = $ 50\% and $p_t = $ 90\% of the data were used for training. As in previous experiments, the joint stationarity attains better recovery. The difference becomes clearer for low SNR, in which case the error is decreased (roughly) by a factor of two w.r.t. the best alternative.

\paragraph{Code.} We remark that our simulations were done using the GSPBOX~\cite{perraudin2014gspbox}, the UNLocBoX~\cite{perraudin2014unlocbox}, and the LTFAT~\cite{ltfatnote030}. The code reproducing our experiments is available at \url{https://lts2.epfl.ch/stationary-time-vertex-signal-processing/}.

\section{Conclusion}

This paper proposed a new definition of wide-sense stationarity appropriate for multivariate processes supported on the vertices of a graph. 

Our model presents two key benefits. 
\emph{First, the estimation and recovery of JWSS processes is efficient, both in terms of estimation variance and computational complexity.} In particular, the JPSD of a JWSS process can be estimated from few observations at a complexity that is roughly linear to the number of graph edges and timesteps. After the PSD has been estimated, the linear MMSE recovery problems of interpolation and denoising can be solved in the same asymptotic complexity. 
\emph{Second, joint stationarity is a volatile model, able to capture non-trivial statistical relations in the temporal and vertex domains.} 
Our experiments suggested that we can model real spatiotemporal processes as jointly stationary without significant loss. Specifically, the JWSS prior was found more expressive than (univariate) TWSS and VWSS priors and improved upon the multivariate time stationarity prior when the dimensionality was large, but the model estimation was based on few observations of the process.

\section{Appendix}

\subsection{Implementation details of the JPSD estimator}
\label{app:fastpsd}

A straightforward implementation requires $\O(N^3)$ operations for computing the eigenbasis of our graph, $\O(N^2 \times KT)$ for performing $KT$ independent GFT, $\O( T\log(T) \times KN )$ for $KN$ independent FFT, and $\O(N^2 T^2)$ for the convolution. 

This section describes how to approximate a convolutional estimate using a number of operations that is linear to $E T$. Before describing the exact algorithm, we note two helpful properties of the estimator. 
First, we can compute $\ddot{h}(\theta)$ by obtaining estimates for each $\X_{(k)}$ independently and then averaging over $k$:   
\begin{align*}
    \dot{h}(\theta) &= \frac{1}{K\, c_g(\theta)} \sum_{k} \sum_{n, \tau} g(\theta - \theta_{n,\tau})^2 \, |\JFT{\X_{(k)}}[n,\tau]|^2
\end{align*}
As we will see in the following, the terms inside the outer sum 
can be approximated efficiently, avoiding the need for an expensive JFT.
In addition, when the convolution window is separable, i.e., $g(\theta) = g_G(\lambda) \cdot g_T(\omega)$, as is assumed in this contribution, the joint convolution can be performed successively (and at any order) in the time and vertex domains
\begin{align}
   \ddot{h}(\theta) 
  & = \sum_{\tau} \frac{g_T(\omega - \omega_{\tau})^2}{c_{g_T}(\omega)} \left( \sum_{n} \frac{g_G(\lambda - \lambda_{n})^2}{c_{g_G}(\lambda)} \, \dot{h}(\theta_{n,\tau}) \right), \nonumber 
\end{align}
where $c_g(\theta) = c_{g_T}(\omega) \cdot c_{g_G}(\lambda)$. Exploiting this property, we treat the implementation of the two convolutions separately and the presented algorithms can be combined in any order. 

\paragraph{Fast time convolution.} This is the textbook case of TPSD estimation, that is solved by the  Welch's method~\cite{welch1967use}. The method entails splitting each time series into equally sized overlapping segments and averaging over segments the squared amplitude of the Fourier coefficients. The procedure is equivalent to an averaging (over time) of the squared coefficients of a Short Time Fourier Transform (STFT), with half-overlapping windows $w_T$ defined such that {$\DFT{w_T(t)} = g_T(\omega)$}~\cite{grochenig2013foundations,feichtinger2012gabor}.
Let $L$ be the support of the autocorrelation, or equivalently the number of frequency bands. We suggest using the iterated sine window 
\begin{align*}
w_T(t) \delequal 
\begin{cases} 
\sin\left(0.5\pi\cos\left(\pi t / L\right)^2\right) &\mbox{if } t\in [-L/2,L/2] \\ 
0 & \mbox{otherwise},
\end{cases}
\end{align*}
as it turns the STFT into a tight operator. In order to get an estimate of $\ddot{h}$ at unknown frequencies, we interpolate between the $L$ known points using splines~\cite{de1978practical}. 

\paragraph{Fast graph convolution.}
Inspired by the technique of~\cite{perraudin2016stationary}, we perform the graph convolution using an approximated graph filtering operation~\cite{susnjara2015accelerated} that scales linearly to the number of graph edges $E$. In particular,
\begin{align}
\sum_{n = 1}^N \frac{g_G(\lambda - \lambda_{n})^2}{c_{g_G}(\lambda)} \, \dot{h}(\theta_{n,\tau}) =  \frac{ \E{\|g_{G}(\lambda\I_N - \LG) \, \x_\tau  \|_2^2} }{c_{g_G}(\lambda)}.
\label{eq:fastgraph_x}
\end{align}
We suggest using the Gaussian window
\begin{align}
  g_{G}(\lambda - \lambda_n) \delequal e^{\displaystyle -{(\lambda - \lambda_n)^2}/{\sigma^2}},
\end{align}
with $\sigma^2 = {2 (F + 1)\lambda_{\text{max}}}/{F^2}$. As we did before, we only compute the above for $F = \O(1)$ different values of $\lambda$ and approximate the rest using splines. As the eigenvalues are not known, we need a stable way to estimate $c_{g_G}(\lambda)$. We obtain an unbiased estimate by filtering $Q = \O(1)$ random Gaussian signals on the graph $\eps \in \Rbb^N \sim \mathcal{N} (0, \I_N)$, such that
\begin{align}
  {c_{g_G}(\lambda)} = \E{ \sum_{q = 1}^Q \|g_{G}(\lambda\I_N - \LG) \eps_{(q)}\|_2^2},
  \label{eq:fastgraph_c}
\end{align}
with variance equal to $2 \sum_{n=1}^N g^4(\lambda - \lambda_n) / Q$. We omit the analysis, as it is similar to that in Theorem~\ref{theo:sample_JPSD}. 
According to our numerical evaluation, the approximation error introduced by the latter estimator and spectral filtering is almost negligible for smooth JPSD.

\paragraph{Complexity.} The computational cost of the above methods is: (a)  $\O( TKF \times E + QF\times E) = \O((TK + Q)EF)$ for the fast graph convolutions. Here, the $TK$ and $Q$ convolutions are performed in order to estimate the quantities at \eqref{eq:fastgraph_x} and \eqref{eq:fastgraph_c} for $F$ different values of $\lambda$. (b) $\O(NK \times T \log(L))$ for the fast time convolution, corresponding to $NK$ STFT. Thus, in total the complexity of the fast convolutional JPSD estimator is $\O(TKFE + QEF + NKT\log(L))$. Furthermore, when $Q,F,K,L$ are constants, the complexity simplifies to $\O(T E)$. We remark that, though asymptotically superior, the fast implementation can be significantly slower when the number of variables is small. Our experiments demonstrate that it should be preferred for $N$ larger than a few thousands (see Figure~\ref{fig:scalability}).

\subsection{Univariate vs multivariate JPSD}
\label{app:univariate-vs-multivariate-PSD}
As discussed in Section~\ref{subsec:relations}, one could potentially pose a VWSS hypothesis on a product graph to define joint stationarity, but the direct effect of such a choice is that the spectral domain becomes 1-dimensional instead of 2-dimensional. To see why this is problematic, in Figure~\ref{fig:comparizon_PSDs} we plot the two different representations of the JPSD for the three datasets featured in our experiments. It can be seen that the 2D representation (corresponding to the JWSS hypothesis) is more structured than its 1D counterpart. More importantly, a JWSS hypothesis leads to a smoother JPSD: this is what our convolutional JPSD estimator employs to decrease the estimation variance.
\begin{figure}[ht!]
\centering
\text{Molene}\\
\includegraphics[width=0.45\columnwidth]{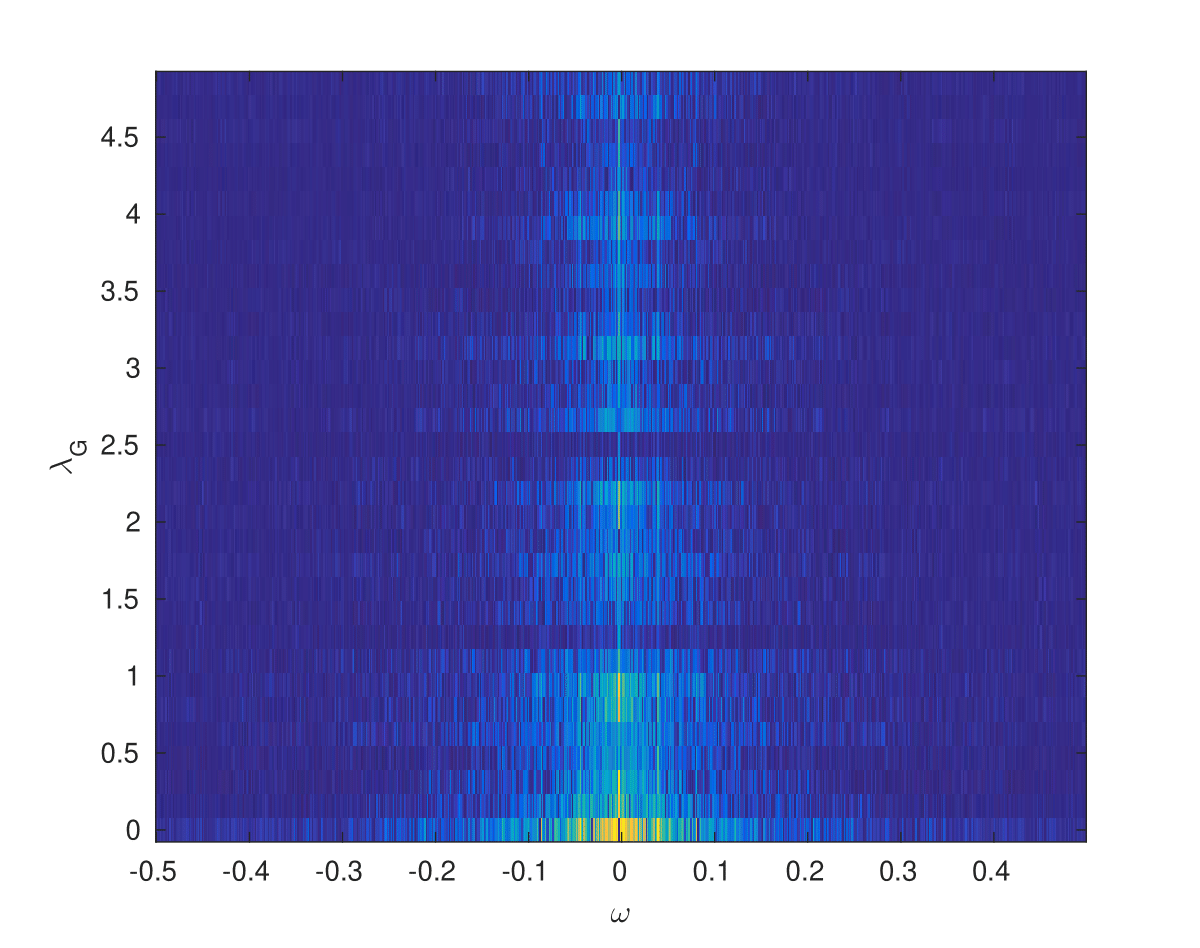}
\includegraphics[width=0.45\columnwidth]{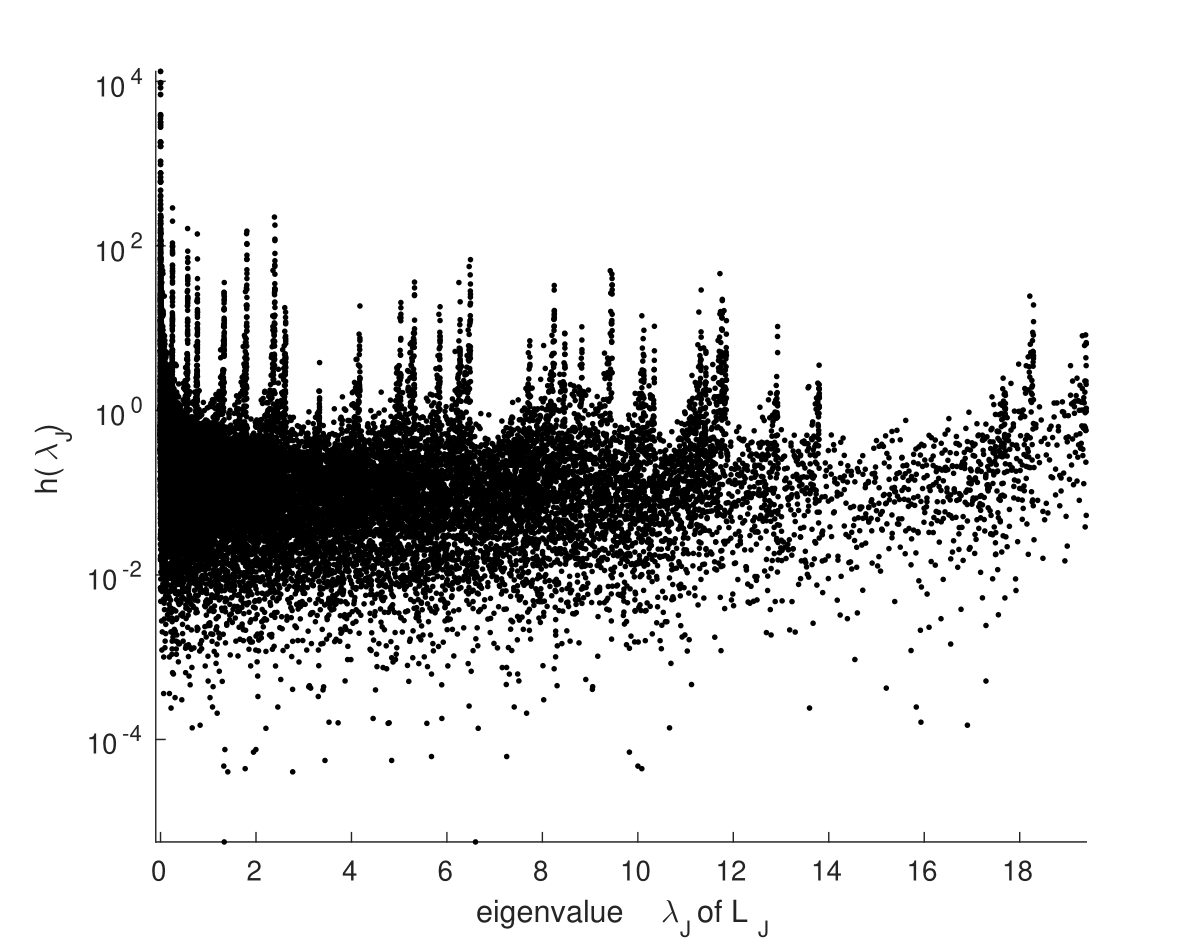}\\
\text{PEMS}\\
\includegraphics[width=0.45\columnwidth]{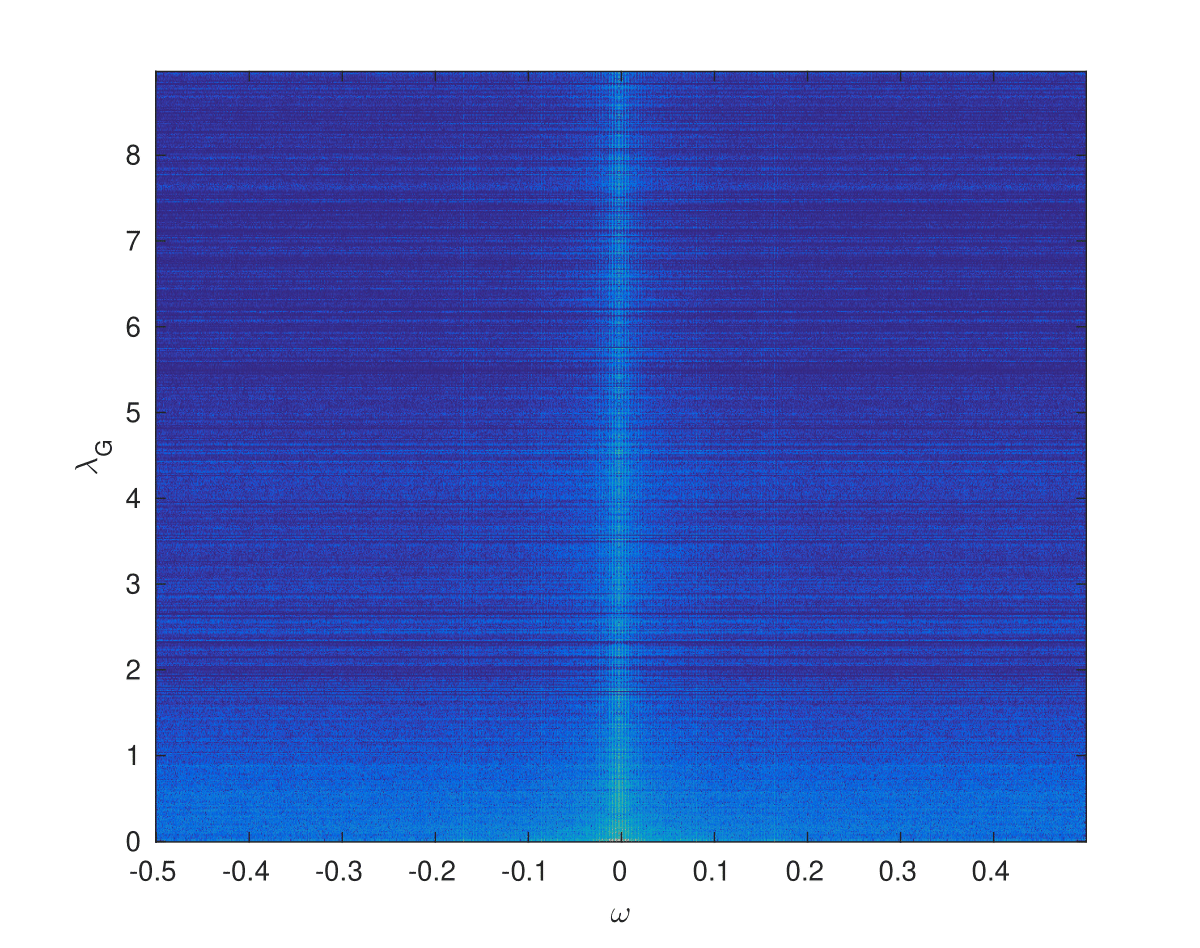}
\includegraphics[width=0.45\columnwidth]{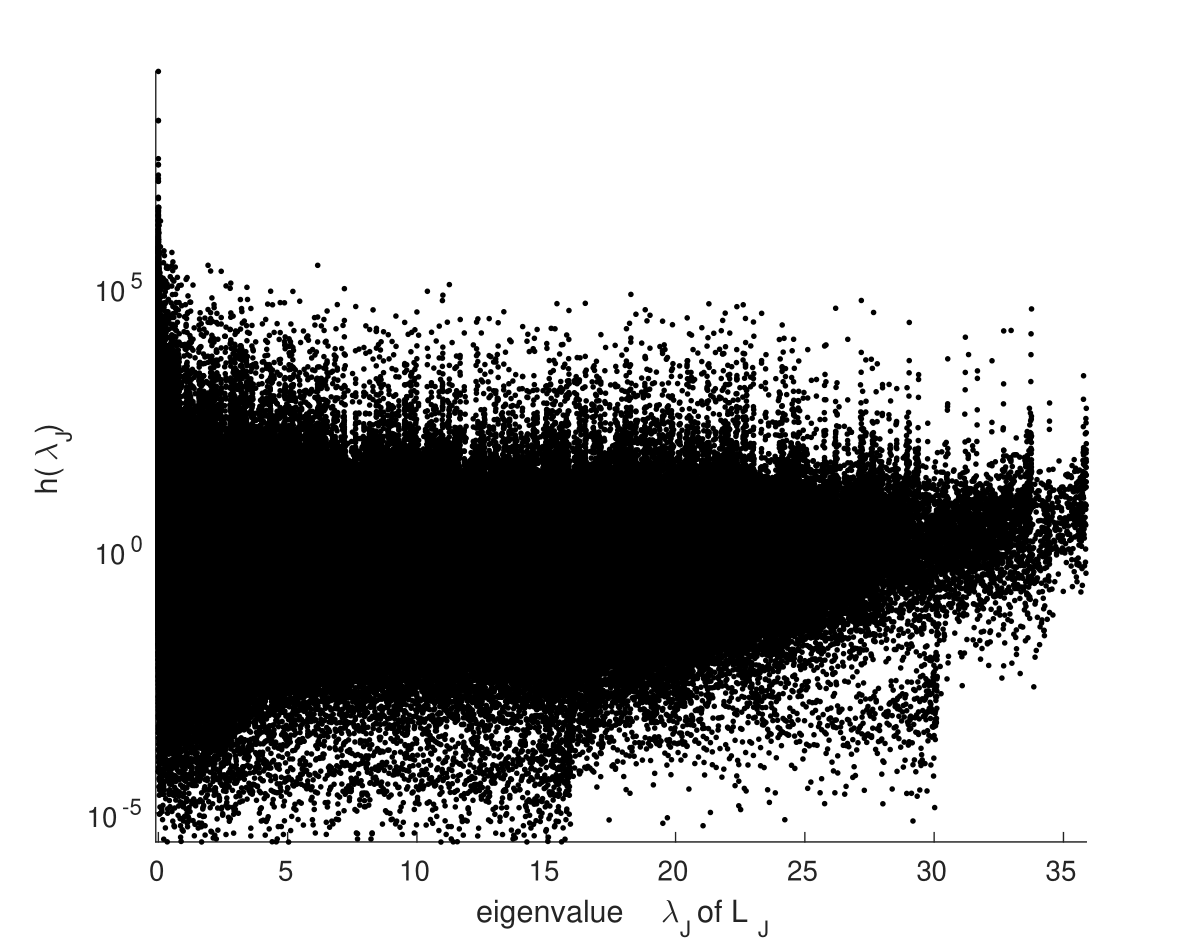}\\
\text{Epidemic}\\
\includegraphics[width=0.45\columnwidth]{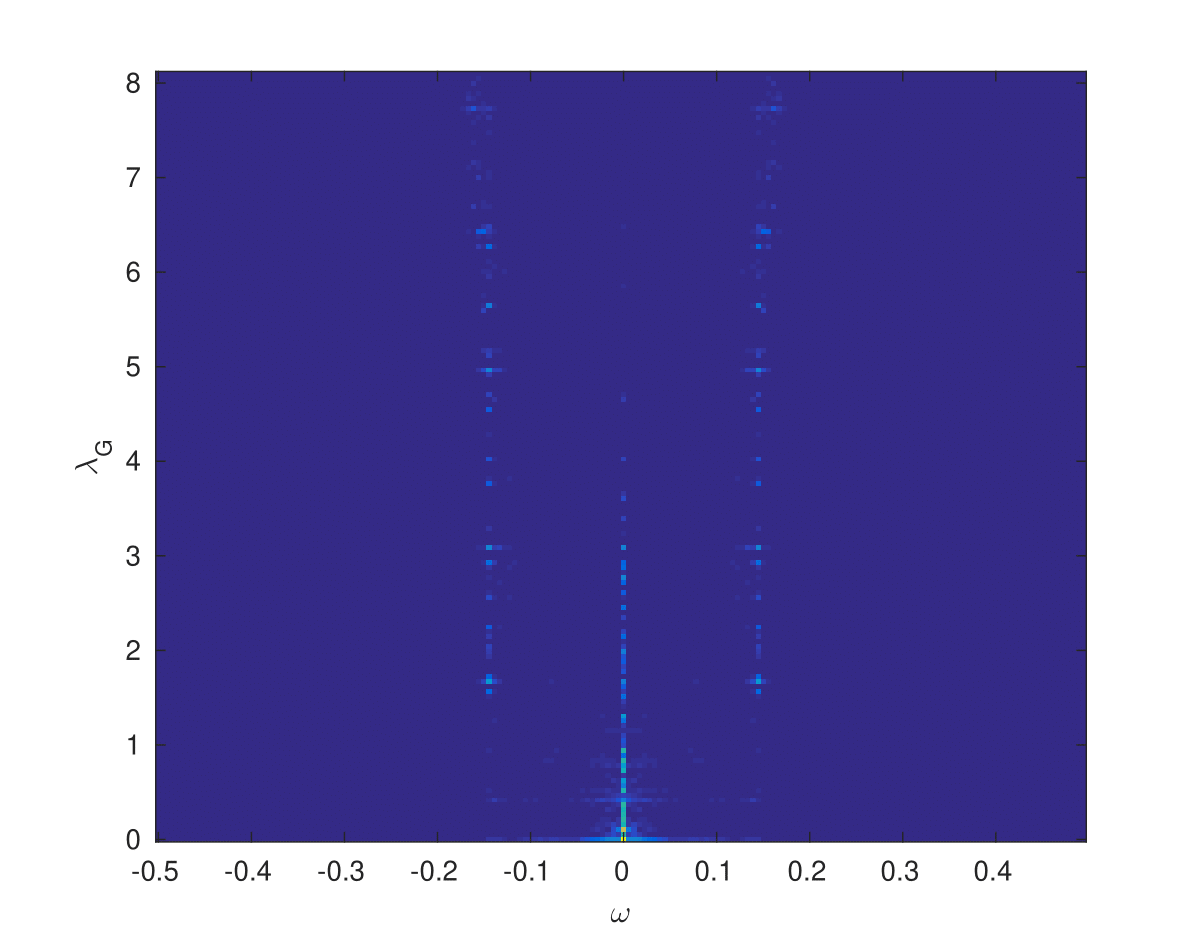}
\includegraphics[width=0.45\columnwidth]{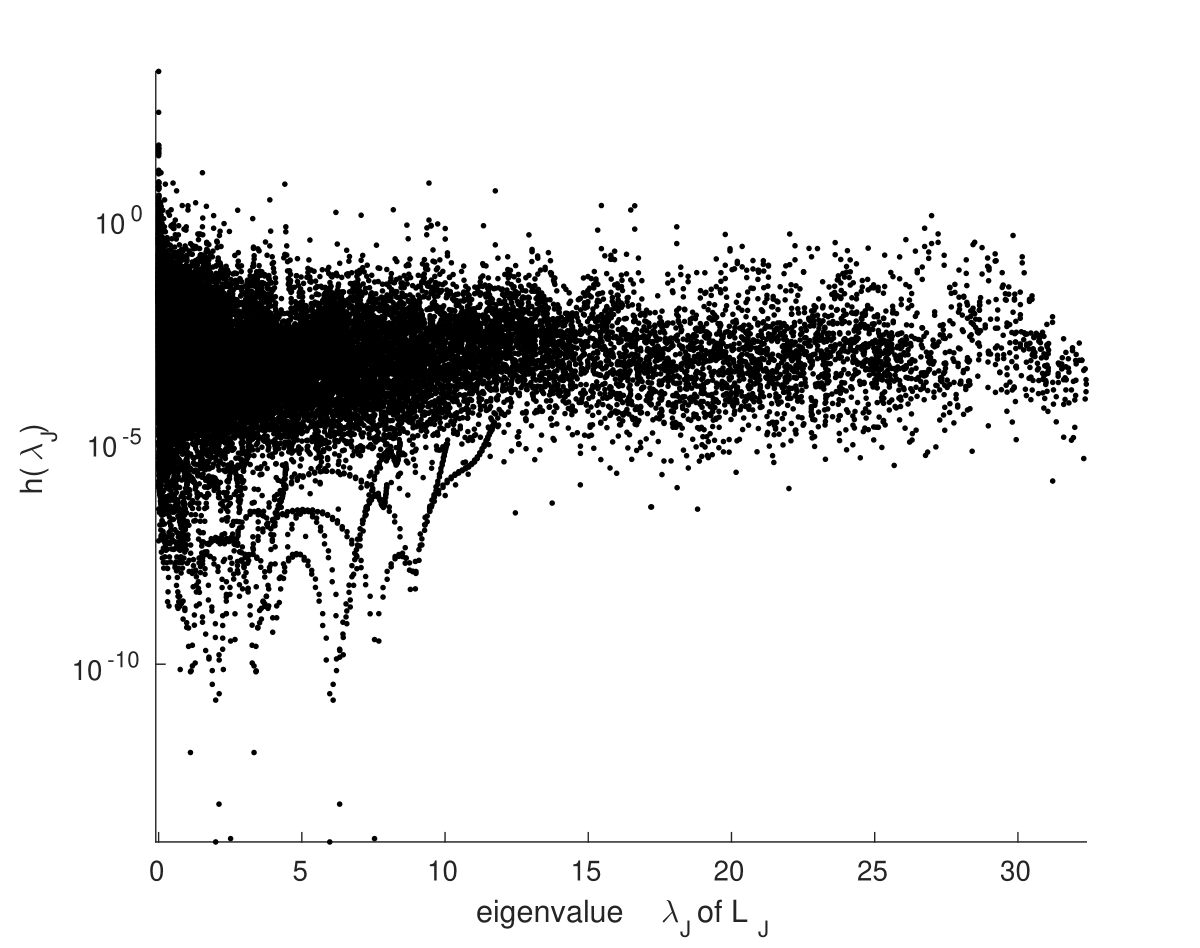}\\
\caption{The JPSD of the three datasets used in our experiments interpreted under a JWSS hypothesis (left) and a VWSS on a product graph hypothesis (right). In both left and right sub-figures, the JPSDs are plotted in logarithmic scale. Interpreting the JPSD as a 2-dimensional function leads to a smoother and more structured representation.
\label{fig:comparizon_PSDs}}
\end{figure}

\subsection{Deferred proofs}

\begin{proof}[Proof of Proposition~\ref{theorem:time-vertex-stationarity}]
In order to simplify the notation in the next proof, we define the unravel function $u_r:\mathbb{Z}^2\rightarrow \mathbb{Z}$ that transforms the double indexes $n,\tau$ of the matrix indexing of $\X$ into its vector index of $u(n,\tau)= (\tau-1)N + n$, i.e., $\X[n,\tau]=\text{vec}(\X)[u(n,\tau)]$.

By construction of the JFT basis, $\hat{\X}[0,0]$ captures the DC-offset of a signal, and condition (a) is equivalent to stating that $\E{\x} = c\1_{NT}$. Moreover, if the graph is connected and (a) holds, at least one of $\E{\hat{\X}[n_1,\tau_1]}$ and $\E{\hat{\X}[n_2,\tau_2]}$ must be zero when ${n_1} \neq {n_2}$ or $\tau_1 \neq {\tau_2}$ and 
\begin{align*}
    \E{\hat{\X}[n_1,\tau_1] \hat{\X}[n_2,\tau_2]}
    &=
\E{\hat{\X}[n_1,\tau_1] \hat{\X}[n_2,\tau_2]} - \E{\hat{\X}[n_1,\tau_1]}\E{\hat{\X}[n_2,\tau_2]} \\
&= (\UJ^* \bSigma \UJ)[u(n_1,\tau_1), u(n_2,\tau_2)]. 
\end{align*}
Therefore, condition (b) is equivalent to stating that $\bSigma = \UJ \b{D} \UJ^\hermitian$ for some diagonal matrix $\b{D}$. In addition, (c) asserts that $\b{D}[u(n,\tau),u(n,\tau)] = h(\lambda_n, \omega_\tau)$ for every $n, \tau$. Thus taken together, (b) and (c) state that $\bSigma = \UJ \b{D} \UJ^\hermitian = \UJ h(\bLambda_G, \bLambda_T) \UJ^\hermitian = h(\LG, \LT)$, which is the second moment condition of a JWSS process.
\end{proof}

\begin{proof}[Proof of Theorem~\ref{theo:time-vertex-def}]
For the first moment, it is straightforward to see that $\E{\X[n,t]}=c$ if and only if both $\E{\X[n,t]} = c_t$ and $\E{\X[n,t]} = c_n$ $\forall n,t$.\\
For the second moment, the covariance matrix of a JWSS process is by definition the linear operator associated to a joint filter $\bSigma = h(\LG,\LT)$. Using~\eqref{eq:joint_filter_detail}, $\bSigma_{t_1,t_2}$ can be written as 
\begin{equation}
\bSigma_{t_1,t_2} = \UG \gamma_\delta(\bLambda) \UG^\hermitian = \gamma_\delta(\LG),
\end{equation}
where $\delta=t_1-t_2+1$ and 
\begin{equation}
    \gamma_\delta(\lambda) = \frac{1}{T} \sum_{\tau=1}^T h(\lambda,\omega_\tau) e^{j \omega_\tau \delta}.
\end{equation}
Hence the process satisfies the (b) statement of Definition~\ref{def:time-stationarity} (TWSS) and~\ref{def:vertex-stationarity} (VWSS).
Conversely if a process is TWSS and VWSS, we have $\bSigma_{t_1,t_2}= \gamma_{t_1,t_2}(\LG) = \gamma_\delta(\LG)$ with the same $\delta$ as before. As a result, using~\eqref{eq:joint_filter_detail}, its covariance matrix can be written as a joint filter $h(\LG,\LT)$, where
\begin{equation}
h(\lambda_n,\omega_\tau) = \sum_{\delta=1}^T \gamma_\delta(\lambda_n) e^{j \omega_\tau \delta},
\end{equation}
and hence also satisfies the property of the second moment of JWSS processes.
\end{proof}

\begin{proof}[Proof of Property~\ref{theo:time-vertex-psd-trans}]
The output of a filter $f(\L_J)$ can be written in vector form as $\y = f(\LJ) $. If the input signal $\x$ is JWSS, we can confirm that the first moment of the filter output is $\E{f(\LJ) \x }= f(\LJ) \E{\x }=f(0,0) \E{\x }$, which remains constant as $\E{\x }$ is constant by hypothesis. The computation of the second moment gives
\begin{align*}
\bSigma_{\y} &= \E{ f(\LJ)\x \left( f(\LJ) \x \right)^* }  - \E{f(\LJ) \x } \E{ (f(\LJ) \x)^* }\\
    &= f(\LJ) \E{ \x \x^* } f(\LJ) - f(\LJ) \E{\x}\E{\x}^* f(\LJ)^*\\ 
    &= f(\LJ) \b{\Sigma}_{\x} f(\LJ)^*  = \UJ \, \left(f^2(\Theta)\, h_\X(\Theta) \right) \, \UJ^*,
\end{align*}
which satisfies the second moment condition of JWSS processes. 
Above, $f^2(\Theta)$ is a diagonal $NT\times NT$ matrix, whose diagonal is obtained by applying the bivariate function $f^2(\cdot, \cdot)$ on $[\lambda_n, \omega_\tau]$ for all $n, \tau$ ($f$ can be interpreted as the frequency response of a joint filter). Matrix $h_\X(\Theta)$ is similarly defined.
\end{proof}

\vspace{-3mm}
\begin{lemma}
If function $h(\theta)$ is $\epsilon$-Lipschitz, then the bias is bounded by
\begin{align*}
\left|\E{\ddot{h}(\theta) - h(\theta)} \right|
\leq \frac{\epsilon }{ c_g(\theta) }\hspace{-1mm} \sum_{n=1,\tau=1}^{T,N}\hspace{-3mm}g(\theta - \theta_{n,\tau})^{2} \norm{\theta - \theta_{n,\tau}}_{2}.
\end{align*}
\label{lemma:convolutional_bias}
\end{lemma}
\begin{proof}
Since $h(\theta)$ is $\epsilon$ Lipschitz, we have $| h(\theta) - h(\theta_{n,\tau}) | \leq \epsilon \norm{\theta - \theta_{n,\tau}}_2 $. Hence, we write
\begin{align*}
  \left|\E{\ddot{h}(\theta) - h(\theta)} \right| 
  & =  \left| \left( \sum_{n,\tau=1}^{NT} \frac{g(\theta- \theta_{n,\tau})^{2}}{c_g(\theta)}  h(\theta_{n,\tau}) \right) - h(\theta)\right| \\
  &\hspace{-27mm}=  \left| \sum_{n,\tau=1}^{NT} \frac{g(\theta- \theta_{n,\tau})^{2}}{c_g(\theta)}  \left( h(\theta_{n,\tau}) - h(\theta) \right)\right| \hspace{-0mm}\leq \frac{\epsilon}{c_g(\theta)} \sum_{n,\tau=1}^{NT} g(\theta- \theta_{n,\tau})^{2} \norm{\theta - \theta_{n,\tau}}_2,
\end{align*}
where the second equality stems from $\sum_{n,\tau}  g^{2}(\theta- \theta_{n,\tau}) = c_g(\theta)$. 
\end{proof}

\vspace{-3mm}
\begin{lemma}
If $\X$ is a JWSS process such that the entries of $\hat{\X}$ are independent random variables, the convolutional JPSD estimate at $\theta$ has variance
\begin{align}
  \var{\ddot{h}(\theta)} = \sum_{n,\tau} \frac{g(\theta- \theta_{n,\tau})^{4}}{c_g(\theta)^2} \, \var{\dot{h}(\theta_{n,\tau})},
\end{align}
where $\var{\dot{h}(\theta_{n,\tau})}$ is the variance of the sample JPSD estimator at $\theta_{n,\tau}$.
\label{lemma:convolutional_variance}
\end{lemma}
\begin{proof} 
Set $$\alpha_{n,\tau}= g(\theta- \theta_{n,\tau})^{2} h(\theta_{n,\tau})/c_g(\theta)$$ and $\hat{\b{E}}_{(k)} = \mat{\hat{\eps}_{(k)}} = \mat{h(\bLambda_G, \bOmega)^{+\sfrac{1}{2}} \hat{\x}_{(k)}}$, where $+$ denotes the pseudo-inverse, $\hat{\eps}_{(k)}$ is white, and $ \mat{\cdot} $ is the matricization operator. The centered random variable
\begin{align*}
  \ddot{h}(\theta) - \E{\ddot{h}(\theta)}  &= \sum_{n,\tau} \frac{g(\theta- \theta_{n,\tau})^{2}}{c_g(\theta)} ( \dot{h}(\theta_{n,\tau}) - h(\theta_{n,\tau})) \\
  &\hspace{-0mm}=  \sum_{n,\tau} \alpha_{n,\tau} \left(\sum_{k} \frac{\hat{\b{E}}_{(k)}[n,\tau] \hat{\b{E}}_{(k)}[n,\tau]^\hermitian }{K} - 1\right) = \sum_{n,\tau} \alpha_{n,\tau} \, z_{n,\tau}
\end{align*}
is a weighted sum of centered, identically distributed random variables $z_{n,\tau}$. Moreover, when the elements of $\hat{\b{E}}_{(k)}$ are independent, so are the variables $z_{n,\tau}$. 
It follows that
\begin{align*}
  \var{\ddot{h}(\theta)} &= \sum_{n,\tau} \alpha_{n,\tau}^2 \, \var{z_{n,\tau}^2} = \sum_{n,\tau} \frac{g(\theta- \theta_{n,\tau})^{4}}{c_g(\theta)^2} \, \var{\dot{h}(\theta_{n,\tau})},  
\end{align*}
which matches our claim.
\end{proof}

\section*{Declarations}

\subsection*{Availability of Data and Materials}
The code to reproduce the results is available at \url{https://lts2.epfl.ch/stationary-time-vertex-signal-processing/}.
\begin{itemize}
  \item Access to the raw Molene dataset is possible directly from \url{https://donneespubliques.meteofrance.fr/donnees_libres/Hackathon/RADOMEH.tar.gz}
  \item The traffic data corresponding to the 3rd district of California and can be downloaded from \url{http://pems.dot.ca.gov/}.
  \item The epidemic dataset is synthetically generated from a SIR model. The network used for the model can be downloaded from \url{https://www.visualizing.org/global-flights-network/}.
\end{itemize}

\subsection*{Funding}
This work has been supported by the Swiss National Science Foundation research project \textit{Towards Signal Processing on Graphs} (grant number: 2000\_21/154350/1).

\subsection*{Author contributions}
The two authors contributed equally both for the experiments and for the writing of the paper.

\subsection*{Acknowledgements}
We thank Francesco Grassi for his help with the code.

\section*{List of abbreviations}
\begin{longtable}{ll}
DFT & Discrete Fourier Transform \\
GFT & Graph Fourier Transform \\
JFT & Joint Fourier Transform \\
PSD & Power Spectral Density \\
TPSD & Time Power Spectral Density \\
VPSD & Vertex Power Spectral Density \\
JPSD & Joint Power Spectral Density \\
WSS & Wide-Sense Stationarity\\
TWSS & Time Wide-Sense Stationarity\\
VWSS & Vertex Wide-Sense Stationarity\\
JWSS & Jointly Wide-Sense Stationary \\
MTWSS & Multivariate Time Wide-Sense Stationary \\
MVWSS & Multivariate Vertex Wide-Sense Stationary \\
AIC & Akaike information criterion\\
\end{longtable}

\bibliographystyle{ieeetr}
\bibliography{biblio}

\begin{thebibliography}{10}

\bibitem{rudelson1998}
M.~Rudelson, ``Random vectors in the isotropic position,'' {\em Journal of
  Functional Analysis}, vol.~164, no.~1, pp.~60--72, 1999.

\bibitem{lutkepohl2005new}
H.~L{\"u}tkepohl, ``New introduction to multiple time series analysis.
  {Springer Science and Business Media,},'' 2005.

\bibitem{ledoit2004well}
O.~Ledoit and M.~Wolf, ``A well-conditioned estimator for large-dimensional
  covariance matrices,'' {\em Journal of multivariate analysis}, vol.~88,
  no.~2, pp.~365--411, 2004.

\bibitem{lam2012factor}
C.~Lam, Q.~Yao, {\em et~al.}, ``Factor modeling for high-dimensional time
  series: inference for the number of factors,'' {\em The Annals of
  Statistics}, vol.~40, no.~2, pp.~694--726, 2012.

\bibitem{connor1995three}
G.~Connor, ``The three types of factor models: A comparison of their
  explanatory power,'' {\em Financial Analysts Journal}, pp.~42--46, 1995.

\bibitem{keeling2005networks}
M.~J. Keeling and K.~T. Eames, ``Networks and epidemic models,'' {\em Journal
  of the Royal Society Interface}, vol.~2, no.~4, pp.~295--307, 2005.

\bibitem{mohan2008nericell}
P.~Mohan, V.~N. Padmanabhan, and R.~Ramjee, ``Nericell: rich monitoring of road
  and traffic conditions using mobile smartphones,'' in {\em Proceedings of the
  6th ACM conference on Embedded network sensor systems}, pp.~323--336, ACM,
  2008.

\bibitem{huang2015graph}
W.~Huang, L.~Goldsberry, N.~F. Wymbs, S.~T. Grafton, D.~S. Bassett, and
  A.~Ribeiro, ``Graph frequency analysis of brain signals,'' {\em IEEE journal
  of selected topics in signal processing}, vol.~10, no.~7, pp.~1189--1203,
  2016.

\bibitem{zhang2008graph}
F.~Zhang and E.~R. Hancock, ``Graph spectral image smoothing using the heat
  kernel,'' {\em Pattern Recognition}, vol.~41, no.~11, pp.~3328--3342, 2008.

\bibitem{smola2003kernels}
A.~J. Smola and R.~Kondor, ``Kernels and regularization on graphs,''
  pp.~144--158, 2003.

\bibitem{belkin2004semi}
M.~Belkin and P.~Niyogi, ``Semi-supervised learning on riemannian manifolds,''
  {\em Machine learning}, vol.~56, no.~1-3, pp.~209--239, 2004.

\bibitem{shuman2013emerging}
D.~I. Shuman, S.~K. Narang, P.~Frossard, A.~Ortega, and P.~Vandergheynst, ``The
  emerging field of signal processing on graphs: Extending high-dimensional
  data analysis to networks and other irregular domains,'' {\em Signal
  Processing Magazine, IEEE}, vol.~30, no.~3, pp.~83--98, 2013.

\bibitem{sandryhaila2013discrete}
A.~Sandryhaila and J.~M. Moura, ``Discrete signal processing on graphs,'' {\em
  IEEE transactions on signal processing}, vol.~61, pp.~1644--1656, 2013.

\bibitem{sandryhaila2014big}
A.~Sandryhaila and J.~M. Moura, ``Big data analysis with signal processing on
  graphs: Representation and processing of massive data sets with irregular
  structure,'' {\em IEEE Signal Processing Magazine}, vol.~31, no.~5,
  pp.~80--90, 2014.

\bibitem{gadde2015probabilistic}
A.~Gadde and A.~Ortega, ``A probabilistic interpretation of sampling theory of
  graph signals,'' in {\em International Conference on Acoustics, Speech and
  Signal Processing (ICASSP)}, pp.~3257--3261, IEEE, 2015.

\bibitem{zhang2015graph}
C.~Zhang, D.~Flor{\^e}ncio, and P.~A. Chou, ``Graph signal processing--a
  probabilistic framework,'' {\em Microsoft Res., Redmond, WA, USA, Tech. Rep.
  MSR-TR-2015-31}, 2015.

\bibitem{perraudin2016stationary}
N.~Perraudin and P.~Vandergheynst, ``{Stationary signal processing on
  graphs},'' {\em IEEE Transactions on Signal Processing}, vol.~65, no.~13,
  pp.~3462 -- 3477, 2017.

\bibitem{girault2015stationary}
B.~Girault, ``Stationary graph signals using an isometric graph translation,''
  in {\em Signal Processing Conference (EUSIPCO), 2015 23rd European},
  pp.~1516--1520, IEEE, 2015.

\bibitem{marques2016stationary}
A.~G. Marques, S.~Segarra, G.~Leus, and A.~Ribeiro, ``Stationary graph
  processes and spectral estimation,'' {\em IEEE Transactions on Signal
  Processing}, vol.~65, no.~22, pp.~5911--5926, 2016.

\bibitem{isufi2017autoregressive}
E.~Isufi, A.~Loukas, A.~Simonetto, and G.~Leus, ``Autoregressive moving average
  graph filtering,'' {\em IEEE Transactions on Signal Processing}, vol.~65,
  no.~2, pp.~274--288, 2017.

\bibitem{loukas2016frequency}
A.~Loukas and D.~Foucard, ``Frequency analysis of time-varying graph signals,''
  in {\em Global Conference on Signal and Information Processing (GlobalSIP)},
  pp.~346--350, IEEE, 2016.

\bibitem{wiener1957prediction}
N.~Wiener and P.~Masani, ``The prediction theory of multivariate stochastic
  processes,'' {\em Acta Mathematica}, vol.~98, no.~1, pp.~111--150, 1957.

\bibitem{wiener1958prediction}
N.~Wiener and P.~Masani, ``The prediction theory of multivariate stochastic
  processes, ii,'' {\em Acta Mathematica}, vol.~99, no.~1, pp.~93--137, 1958.

\bibitem{bloomfield2004fourier}
P.~Bloomfield, ``Fourier analysis of time series: an introduction. {John Wiley
  \& Sons},,'' 2004.

\bibitem{bach2004learning}
F.~R. Bach and M.~I. Jordan, ``Learning graphical models for stationary time
  series,'' {\em IEEE transactions on signal processing}, vol.~52, no.~8,
  pp.~2189--2199, 2004.

\bibitem{dahlhaus2003causality}
R.~Dahlhaus and M.~Eichler, ``Causality and graphical models in time series
  analysis,'' {\em Oxford Statistical Science Series}, pp.~115--137, 2003.

\bibitem{girault2015signal}
B.~Girault, {\em Signal Processing on Graphs-Contributions to an Emerging
  Field}.
\newblock PhD thesis, Ecole normale sup{\'e}rieure de lyon, 2015.

\bibitem{chepuri2016subsampling}
S.~P. Chepuri and G.~Leus, ``Subsampling for graph power spectrum estimation,''
  in {\em Sensor Array and Multichannel Signal Processing Workshop (SAM)},
  pp.~1--5, IEEE, 2016.

\bibitem{loukas2016predicting}
A.~Loukas, E.~Isufi, and N.~Perraudin, ``Predicting the evolution of stationary
  graph signals,'' in {\em Asilomar Conference on Signals, Systems, and
  Computers}, pp.~60--64, IEEE, 2017.

\bibitem{mei2015signal}
J.~Mei and J.~M. Moura, ``Signal processing on graphs: Causal modeling of
  unstructured data,'' {\em IEEE Transactions on Signal Processing}, vol.~65,
  no.~8, pp.~2077--2092, 2017.

\bibitem{ioannidis2018inference}
V.~N. Ioannidis, D.~Romero, and G.~B. Giannakis, ``Inference of spatio-temporal
  functions over graphs via multikernel kriged kalman filtering,'' {\em
  Transactions on Signal Processing}, vol.~66, no.~12, pp.~3228--3239, 2017.

\bibitem{8081618}
P.~D. Lorenzo, E.~Isufi, P.~Banelli, S.~Barbarossa, and G.~Leus, ``Distributed
  recursive least squares strategies for adaptive reconstruction of graph
  signals,'' in {\em 2017 25th European Signal Processing Conference
  (EUSIPCO)}, pp.~2289--2293, Aug 2017.

\bibitem{perraudin2016towards}
N.~Perraudin, A.~Loukas, F.~Grassi, and P.~Vandergheynst, ``Towards stationary
  time-vertex signal processing,'' {\em IEEE International Conference on
  Acoustics, Speech and Signal Processing (ICASSP)}, 2017.

\bibitem{grassi2017timevertex}
F.~Grassi, A.~Loukas, N.~Perraudin, and B.~Ricaud, ``A time-vertex signal
  processing framework,'' {\em IEEE Transactions on Signal Processing}, 2017.

\bibitem{loukas2015distributed}
A.~Loukas, A.~Simonetto, and G.~Leus, ``Distributed autoregressive moving
  average graph filters,'' {\em IEEE Signal Processing Letters}, vol.~22,
  no.~11, pp.~1931--1935, 2015.

\bibitem{isufi2016separable}
E.~Isufi, A.~Loukas, A.~Simonetto, and G.~Leus, ``Separable autoregressive
  moving average graph-temporal filters,'' in {\em Signal Processing Conference
  (EUSIPCO), 2016 24th European}, pp.~200--204, IEEE, 2016.

\bibitem{shuman2016vertex}
D.~I. Shuman, B.~Ricaud, and P.~Vandergheynst, ``Vertex-frequency analysis on
  graphs,'' {\em Applied and Computational Harmonic Analysis}, vol.~40, no.~2,
  pp.~260--291, 2016.

\bibitem{segarra2018statistical}
S.~Segarra, S.~P. Chepuri, A.~G. Marques, and G.~Leus, ``Statistical graph
  signal processing: Stationarity and spectral estimation,'' pp.~325--347,
  2018.

\bibitem{vershynin2012}
R.~Vershynin, ``How close is the sample covariance matrix to the actual
  covariance matrix?,'' {\em Journal of Theoretical Probability}, vol.~25,
  no.~3, pp.~655--686, 2012.

\bibitem{bartlett1950periodogram}
M.~S. Bartlett, ``Periodogram analysis and continuous spectra,'' {\em
  Biometrika}, pp.~1--16, 1950.

\bibitem{welch1967use}
P.~Welch, ``The use of fast fourier transform for the estimation of power
  spectra: a method based on time averaging over short, modified
  periodograms,'' {\em IEEE Transactions on audio and electroacoustics},
  pp.~70--73, 1967.

\bibitem{axelsson1980conjugate}
O.~Axelsson, ``Conjugate gradient type methods for unsymmetric and inconsistent
  systems of linear equations,'' {\em Linear algebra and its applications},
  vol.~29, pp.~1--16, 1980.

\bibitem{combettes2005signal}
P.~L. Combettes and V.~R. Wajs, ``Signal recovery by proximal forward-backward
  splitting,'' {\em Multiscale Modeling \& Simulation}, vol.~4, no.~4,
  pp.~1168--1200, 2005.

\bibitem{combettes2011proximal}
P.~L. Combettes and J.-C. Pesquet, ``Proximal splitting methods in signal
  processing,'' pp.~185--212, 2011.

\bibitem{komodakis2015playing}
N.~Komodakis and J.-C. Pesquet, ``Playing with duality: An overview of recent
  primal? dual approaches for solving large-scale optimization problems,'' {\em
  IEEE Signal Processing Magazine}, vol.~32, no.~6, pp.~31--54, 2015.

\bibitem{combettes2007douglas}
P.~L. Combettes and J.-C. Pesquet, ``A douglas--rachford splitting approach to
  nonsmooth convex variational signal recovery,'' {\em IEEE Journal of Selected
  Topics in Signal Processing}, vol.~1, no.~4, pp.~564--574, 2007.

\bibitem{akaike1974new}
H.~Akaike, ``A new look at the statistical model identification,'' {\em IEEE
  transactions on automatic control}, vol.~19, no.~6, pp.~716--723, 1974.

\bibitem{perraudin2014gspbox}
N.~{Perraudin}, J.~{Paratte}, D.~{Shuman}, V.~{Kalofolias}, P.~{Vandergheynst},
  and D.~K. {Hammond}, ``{GSPBOX: A toolbox for signal processing on graphs},''
  {\em ArXiv e-prints}, Aug. 2014.

\bibitem{perraudin2014unlocbox}
N.~{Perraudin}, D.~{Shuman}, G.~{Puy}, and P.~{Vandergheynst}, ``{UNLocBoX A
  matlab convex optimization toolbox using proximal splitting methods},'' {\em
  ArXiv e-prints}, Feb. 2014.

\bibitem{ltfatnote030}
Z.~Prusa, P.~L. Sondergaard, N.~Holighaus, C.~Wiesmeyr, and P.~Balazs, ``{The
  Large Time-Frequency Analysis Toolbox 2.0},'' pp.~419--442, 2014.

\bibitem{grochenig2013foundations}
K.~Gr{\"o}chenig, ``Foundations of time-frequency analysis. {Springer Science
  \& Business Media,},'' 2013.

\bibitem{feichtinger2012gabor}
H.~G. Feichtinger and T.~Strohmer, ``Gabor analysis and algorithms: Theory and
  applications. {Springer Science \& Business Media},,'' 2012.

\bibitem{de1978practical}
C.~De~Boor, ``A practical guide to splines. {Springer-Verlag New York},,''
  vol.~27, 1978.

\bibitem{susnjara2015accelerated}
A.~Susnjara, N.~Perraudin, D.~Kressner, and P.~Vandergheynst, ``Accelerated
  filtering on graphs using lanczos method,'' {\em arXiv preprint
  arXiv:1509.04537}, 2015.

\end{thebibliography}

\end{document}